\documentclass[letterpaper, 10 pt, conference]{IEEEtran}

\IEEEoverridecommandlockouts
\usepackage{amsfonts}
\usepackage{amsmath}
\usepackage{amssymb}
\usepackage{algorithm}
\usepackage{algorithmic}
\usepackage{cite,color,comment,xspace}
\usepackage{times}
\usepackage{mathptmx} %Times font as default roman. Math in Times where possible.
\usepackage{enumerate}
\usepackage{graphicx}
\usepackage{subfigure}
\usepackage{epstopdf}
\usepackage{epsfig}
\usepackage{url}
\usepackage[pdftex,bookmarks=true]{hyperref}
\hypersetup{bookmarks, colorlinks=true, linkcolor=red, urlcolor=blue, citecolor=green}
\usepackage{makeidx}         % allows index generation
\usepackage{multicol}        % used for the two-column index
\usepackage[bottom]{footmisc}% places footnotes at page bottom
\usepackage{latexsym}
\usepackage{textcomp}

\graphicspath{{fig/}}

\newtheorem{theorem}{\bf Theorem}[section]
\newtheorem{problem}{\bf Problem}[section]

\newtheorem{definition}{\bf Definition}[section]
\newtheorem{remark}{\bf Remark}[section]

\newcommand{\be}{\begin{equation}}
\newcommand{\ee}{\end{equation}}
\newcommand{\ben}{\begin{equation*}}
\newcommand{\een}{\end{equation*}}
\newcommand{\bea}{\begin{eqnarray}}
\newcommand{\eea}{\end{eqnarray}}
\newcommand{\bean}{\begin{eqnarray*}}
\newcommand{\eean}{\end{eqnarray*}}
\newcommand{\ba}{\begin{array}}
\newcommand{\ea}{\end{array}}
\newcommand{\leftm}{\left[\begin{array}}
\newcommand{\rightm}{\end{array}\right]}

\newcommand{\ie}{{\it i.e. }}

\makeindex             % used for the subject index
\title{\LARGE \bf Multi-robot deployment from LTL specifications \\ with reduced communication \\ {\rm - technical report -}}

\author{Marius Kloetzer, Xu Chu Ding, and Calin Belta% <-this % stops a space
\thanks{This work was partially supported by the CNCS-UEFISCDI grant PN-II-RU PD 333/2010 at the Technical University of Iasi, Romania, and by ONR MURI N00014-09-1051, ARO W911NF-09-1-0088, AFOSR YIP FA9550-09-1-020 and NSF CNS-0834260 at Boston University.}% <-this % stops a space
\thanks{M. Kloetzer is with the Department of Automatic Control and Applied Informatics, Technical University ``Gheorghe Asachi" of Iasi, Romania
        {\tt\small kmarius@ac.tuiasi.ro}}%
\thanks{X. C. Ding and C. Belta are with the Department of Mechanical Engineering, Boston University, Boston, MA 02215, USA
        {\tt\small \{xcding,cbelta\}@bu.edu}}%
}

\begin{document}

\maketitle \thispagestyle{empty} \pagestyle{empty}

%%%%%%%%%%%%%%%%%%%%%%%%%%%%%%%%%%%%%%%%%%%%%%%%%%%%%%%%%%%%%%%%%%%%%%%%%%%%%%%%
\begin{abstract}
In this work, we develop a computational framework for fully automatic
deployment of a team of unicycles from a global specification given as an LTL
formula over some regions of interest. Our hierarchical approach consists of
four steps: (i) the construction of finite abstractions for the motions of
each robot, (ii) the parallel composition of the abstractions, (iii) the
generation of a satisfying motion of the team; (iv) mapping this motion to
individual robot control and communication strategies. The main result of the
paper is an algorithm to reduce the amount of inter-robot communication
during the fourth step of the procedure.
\end{abstract}

%%%%%%%%%%%%%%%%%%%%%%%%%%%%%%%%%%%%%%%%%%%%%%%%%%%%%%%%%%%%%%%%%%%%%%%%%%%%%%%%
\section{INTRODUCTION}

Motion planning and control is a fundamental problem that have
been extensively studied in the robotics literature~\cite{SML:06}.
Most of the existing works have focused on point-to-point
navigation, where a mobile robot is required to travel from an
initial to a final point or region, while avoiding obstacles.
Several solutions have been proposed for this problem, including
cell decomposition based approaches that use graph search
algorithms such as $A^*$ \cite{SR-PN:03,SML:06},  continuous
approaches involving navigation functions and potential
fields~\cite{ER-DEK:92}, and sampling-based methods such as
Rapidly-Exploring Random Trees
(RRTs)~\cite{SML-JJK:01,RT-IRM-MMT-JWR:10}.  However, the above
approaches cannot accommodate complex task specifications, where a
robot might be required to satisfy some temporal and logic
constraints, {\it e.g.,} ``avoid $E$ for all times; visit $A$ or
$B$ and then be at either $C$ or $D$ for all times".

In recent years, there has been an increased interest in using temporal
logics to specify mission plans for robots
\cite{Antoniotti95,Karaman_mu_09,KB-TAC08-LTLCon,Hadas-ICRA07,Loizou04,Quottrup04,Tok-Ufuk-Murray-CDC09}.
Temporal logics \cite{baier2008principles,Clarke99,automata-book07} are
appealing because they provide formal, rich, and high level languages in
which to describe complex missions. For example, the above task specification
translates immediately to the Linear Temporal Logic (LTL) formula
$\square\neg E \wedge \diamondsuit ((A \vee B) \wedge \diamondsuit\square (C
\vee D))$, where $\neg$, $\vee$, $\wedge$ are the usual Boolean operators,
and $\diamondsuit$, $\square$ are two temporal operators standing for
``eventually" and ``always", respectively. Computation Tree Logic (CTL) and
$\mu$-calculus have also been advocated as robot motion specification
languages \cite{Quottrup04,Karaman_mu_09}.

To use formal languages and model checking techniques for robot
motion planning and control, a fundamental challenge is to
construct finite models that accurately capture the robot motion
and control capabilities. Most current approaches are based on the
notion of abstraction \cite{Alur00}. Enabled by recent
developments in hierarchical abstractions of dynamical systems
\cite{belta2006controlling,burridge1999sequential,conner2006integrated,DOK-ICRA98,
HabColSchup06}), it is now possible to model the motions of
several types of robots as finite transition systems over a
cell-based decomposition of the environment. By using equivalence
relations such as simulations and bisimulations \cite{Milner89},
the motion planning and control problem can be reduced to a model
checking or formal synthesis problem for a finite transition
system, for which several techniques are readily available
\cite{VW86,Holzmann97,Clarke99,DiVinE}.

Some recent works suggest that such single-robot techniques can be
extended to multi-robot systems through the use of parallel
composition (synchronous products) \cite{Quottrup04,KB-TRO-2009}.
The main advantage of such a bottom-up approach is that the motion
planning problem can be solved by off-the-shelf model checking on
the parallel composition followed by canonical projection on the
individual transition systems. The two main limitations, both
caused by the parallel composition, are the state space explosion
problem and the need for inter-robot synchronization
(communication) every time a robot leaves its current region. In
our previous work, we proposed bisimulation-type techniques to
reduce the size of the synchronous product in the case when the
robots are identical \cite{KB-ICNSC-06} and derived classes of
specifications that do not require any inter-robot communication
\cite{KB-TRO-2009}. By drawing inspiration from distributed formal
synthesis \cite{mukund2002}, we have also proposed top-down
approaches that do not require the parallel composition of the
individual transition systems \cite{yushandars}. While cheaper,
this method restricts the specifications to regular expressions.

In this paper, we focus on bottom-up approaches based on parallel composition
and address one of the limitations mentioned above. Specifically, we develop
an algorithm that determines a reduced number of synchronization (communication) moments along
a satisfying run of the parallel composition. Our approach is heuristic - we
do not minimize the necessary amount of communication. However, our approach
can be directly modified to produce minimal sets of synchronization moments.
Our extensive experiments show that the proposed   algorithm leads to a significant
reduction in the number of synchronizations. We integrate this algorithm into
a software tool for automatic deployment of unicycles with polyhedral control
constraints from specifications given as LTL formulas over the regions of an
environment with a polyhedral partition. The user friendly tool, which is
freely downloadable from \url{http://hyness.bu.edu/\~software/MRRC.htm},
takes as input a user-defined environment, an LTL formula over some
polytopes, the number of unicycles, and their forward and angular velocity
constraints. It returns a control and communication strategy for each robot
in the team. While transparent to the user, the tool also implements
triangulation and polyhedral operation algorithms from
\cite{triangle-soft,cdd-soft,KB-TAC08-LTLCon}, LTL to B\"uchi conversion
\cite{Gastin01}, and robot abstraction by combining the affine vector field
computation from \cite{HabColSchup06} with input-output regulation
\cite{DOK-ICRA98}.

The remainder of the paper is organized as follows. Sec.
\ref{sec:prelim}presents some preliminary notions necessary throughout the
paper. Sec. \ref{sec:prob_form} formulates the general problem we want to
solve, outlines the solution and then presents a specific problem of
interest. The main contribution of the paper is given in Sec.
\ref{sec:find_synch_mom}. To illustrate the method and the main concepts, a
case study is examined throughout the paper and concluded in Sec.
\ref{sec:case_study_solved}, while an additional case study is included in
Sec. \ref{sec:examples}. The paper ends with conclusions and final remarks in
Sec. \ref{sec:concl}.

\section{PRELIMINARIES}\label{sec:prelim}

\begin{definition}\label{def:tran_syst}
A deterministic finite transition system is a tuple
$T=(Q,q_0,\rightarrow,\Pi,\rho)$, where $Q$ is a (finite) set of
states, $q_0\in Q$ is the initial state, $\rightarrow\subseteq
Q\times Q$ is a transition relation, $\Pi$ is a finite set of
atomic propositions (observations), and $\rho: Q\rightarrow 2^\Pi$
is the observation map.
\end{definition}

To avoid supplementary notations, we do not include control inputs in
Definition \ref{def:tran_syst} since $T$ is deterministic, \ie we can choose
any available transition at a given state. A {\em trajectory} or {\em run} of
$T$ starting from $q$ is an infinite sequence $r=r(1)r(2)r(3)\ldots$ with the
property that $r(1)=q_0$, $r(i)\in Q$, and $(r(i),r(i+1))\in\rightarrow$,
$\forall i\ge 1$. A trajectory $r=r(1)r(2)r(3)\ldots$ defines an infinite
{\em word} over set $2^\Pi$, $w=w(1)w(2)w(3)\ldots$, where $w(i)=\rho(r(i))$.
With a slight abuse of notation, we will denote by $\rho(r)$ the word
generated by run $r$. The set of all words that can be generated by $T$ is
called the ($\omega$-) language of $T$.

In this paper we consider motion specifications given as formulas of Linear
Temporal Logic (LTL) \cite{Clarke99}.  A formal definition for the syntax and
semantics of LTL formulas is beyond the scope of this paper. Informally, the
LTL formulas are recursively defined over a set of atomic propositions $\Pi$,
by using the standard boolean operators and a set of temporal operators. The
boolean operators are $\neg$ (negations), $\vee$ (disjunction), $\wedge$
(conjunction), and the temporal operators that we use are: $U$ (standing for
``until"), $\square$ (``always"), $\diamondsuit$ (``eventually"). LTL
formulas are interpreted over infinite words over set $2^\Pi$, as are those
generated by transition system $T$.  For LTL formulas satisfied by continuous
systems, we restrict the class of specifications to $\rm{LTL}_{-X}$, which
are LTL formulas without the ``next" temporal operator.  We note that, the
class of $\rm{LTL}_{-X}$ is not at all restrictive, since for continuous
systems $\rm{LTL}_{-X}$ captures the full expressivity power of LTL
\cite{KB-TAC08-LTLCon}.

All LTL formulas can be converted into a generalized B{\"u}chi automaton
\cite{Gastin01} defined below:

\begin{definition}[Generalized B{\"u}chi automaton]\label{def:buchi_generalized}
A generalized B{\"u}chi automaton is a tuple
$\mathcal{B}=(S,s_0,\Sigma,\rightarrow_\mathcal{B},F)$, where
\begin{itemize}
\item $S$ is a finite set of states, \item $S_0\subseteq S$ is the
set of initial states, \item $\Sigma$ is the input alphabet, \item
$\rightarrow_\mathcal{B}\subseteq S\times \Sigma\times S $ is a
(nondeterministic) transition relation, \item $F\subseteq 2^S$ is
the {\it set of sets} of accepting (final) states.
\end{itemize}
\end{definition}

The semantics of a B{\"u}chi automaton is defined over infinite input words
over $\Sigma$. An input word is accepted by automaton $\mathcal{B}$ if and
only if there exists a run produced by that word with the property that all
sets from $F$ are infinitely often visited. Due to the complicated acceptance
condition (multiple sets have to be infinitely often visited), a generalized
B{\"u}chi automaton is usually converted into a regular (degeneralized)
B{\"u}chi automaton. A regular B{\"u}chi automaton has only one set of final
states, {\it i.e.} $F\in 2^S$. Any generalized B{\"u}chi automaton can be
transformed into a regular B{\"u}chi automaton that accepts the exact same
words.  A conversion algorithm can be found in \cite{Gastin01}.

For any LTL formula $\phi$ over a set of atomic propositions
$\Pi$, there exists a (generalized or regular) B{\"u}chi automaton
$\mathcal{B}_\phi$ with input alphabet $\Sigma\subseteq 2^\Pi$
accepting {\it all and only} infinite words over $\Pi$
satisfying formula $\phi$ \cite{Wolper83}.

Given a transition system $T$ with set of observations $\Pi$ and an LTL
formula $\phi$ over $\Pi$, one can find a trajectory of $T$ which generates a
word satisfying $\phi$. This can be done by using model checking inspired
tools \cite{KB-TAC08-LTLCon}, which begin by translating $\phi$ to a regular
B{\"u}chi automaton $\mathcal{B}_\phi$. Then, the product of $T$ with
$\mathcal{B}_\phi$ is computed, operation that can be informally viewed as a
matching between observations of $T$ and transitions of $\mathcal{B}_\phi$.
An accepted run (if any) of the obtained product automaton is chosen. This
accepted run is projected into a run $r$ of $T$, which generates a word
satisfying $\phi$. Although $r$ is infinite, it has a finite-representable
form, namely it consists of a finite string called {\it prefix}, followed by
infinite repetitions of another finite string called {\it suffix} (such a run
is said to be in the prefix-suffix form). A cost criterion can be imposed on
the obtained run $r$, {\it e.g.} the minimum memory for storing it, or a
minimum cost on transitions of $T$ encountered when following the prefix and
a finite number of repetitions of the suffix. The run $r$ can be be directly
generated on $T$ if one can deterministically control (impose) the transition
that appears in each state (which is true for a deterministic transition
system $T$ defined in \ref{def:tran_syst}).

\section{PROBLEM FORMULATION AND APPROACH}\label{sec:prob_form}

In this paper we are interested in developing an automated framework for
deploying identical unicycle robots in planar environments. Assume that such
a robot is described by $(x,y,\theta)$, where $(x,y)\in \mathbb{R}^2$ gives
the position vector of the robot's center of rotation, and $\theta$ is its
orientation. The control $w=[v,\omega]^T\in W\subseteq \mathbb{R}^2$ consists
of forward driving ($v$) and steering ($\omega$) speeds, where $W$ is a set
capturing control bounds. We assume that $W$ includes an open ball around the
origin, which implies that $v$ can be also negative (\ie the robot can drive
backwards). The kinematics of the unicycle are given by:
\begin{equation}\label{eqn:unicycle_kinematics}
\left\{
\begin{array}{l}
\dot{x}=v\cos\theta\\
\dot{y}=v\sin\theta\\
\dot{\theta} = \omega
\end{array}
\right.
\end{equation}

Assume that some robots move in a polygonal convex environment with
kinematics given by \eqref{eqn:unicycle_kinematics}, where a set $\Pi$ of
non-overlapping convex polygonal regions\footnote{Note that convex
non-polygonal regions can be bounded by convex polygonal regions with
arbitrary accuracy, and non-convex regions can be divided into adjacent
convex regions} are defined. The general deployment problem for a team of
unicycle robots satisfying an LTL formula is given by:

\begin{problem}\label{pr:unicycles}
Given a team of $n$ unicycles and a task in the form of an $\rm{LTL}_{-X}$
formula $\phi$ over a set of regions of interest $\Pi$, design individual
communication and control strategies for the mobile robots such that the
motion of the team satisfies the specification.
\end{problem}

We assume that the unicycles are identical and have a small (negligible) size
when compared to the size of the environment and of the predefined regions.
Moreover, we consider that a unicycle visits (or avoids) a region when a
specific reference point on it visits (or avoids) that region. To fully
understand Problem \ref{pr:unicycles}, we say that the motion of the team
satisfies an LTL formula $\phi$ if the word generated during the motion
satisfies $\phi$. The word generated by a set of $n$ continuous trajectories
is a straightforward generalization of the definition of the word generated
by a single trajectory \cite{KB-TAC08-LTLCon}. Informally, the word generated
by the team motion consists of a sequence of elements of $2^\Pi$ containing
the satisfied propositions (visited regions) as time evolves. In a generated
word, there are no finite successive repetition of the same element of
$2^\Pi$, and infinite successive repetitions of the same element appear if
and only if each robot trajectory stays inside a region.

{\bf Case study}: For better understanding of introduced concepts, throughout
this paper, we consider a case study with the environment illustrated by Fig.
\ref{fig:environment}, where 3 unicycle-type mobile robots evolve. The
reference point of each unicycle is its ``nose", the center of rotation is
the middle of the rear axis, and the initial deployment of robots is the one
from Fig. \ref{fig:environment}. The imposed specification requires that
``regions $\pi_1$ and $\pi_4$ and $\pi_6$ are simultaneously visited, and
regions $\pi_2$ and $\pi_5$ are simultaneously visited, infinitely often,
while region $\pi_3$ is always avoided". This specification translates to the
following LTL formula:
\begin{equation}\label{eqn:specification}
\phi= \square\neg\pi_3 \wedge \square\diamondsuit\left( (\pi_1
\wedge \pi_4 \wedge \pi_6) \wedge \diamondsuit (\pi_2 \wedge
\pi_5) \right)
\end{equation}
\endproof

\begin{figure}
\center
\includegraphics[width=0.7\columnwidth]{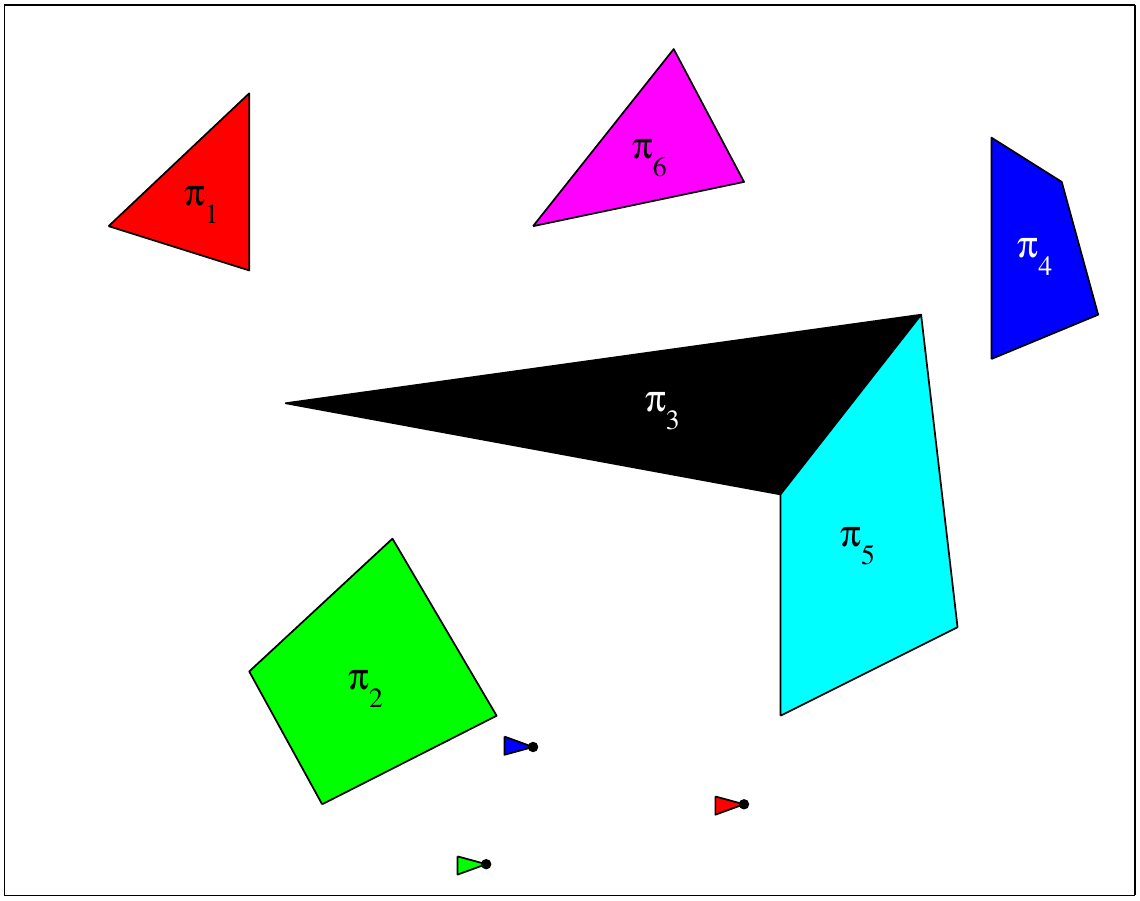}
\caption{A polygonal environment, six regions of interest, and the initial
deployment of three unicycle robots. Robot 1 is green, robot 2 is blue, and
robot 3 is red; there is no relation between the robot and the region
colors.} \label{fig:environment}
\end{figure}

To provide a deployment strategy for Problem \ref{pr:unicycles}, we will
first combine various techniques from computational geometry, motion planning
and model checking until we obtain a solution in the form of a sequence of
tuples of smaller regions and feedback control laws in each of these regions.
After this, we focus on the main contribution of the paper, namely finding a
reduced set of communication (synchronization) moments among robots, while
still guaranteeing the satisfaction of the specification. The main steps of
the algorithmic approach for solving Problem \ref{pr:unicycles} are given in
the following 3 subsections.

\subsection{Robot Abstraction}\label{sec:rob_abstraction}

We first abstract the motion capabilities of each robot to a finite
transition system. To this end, the environment is first partitioned into
convex regions (cells) such that two adjacent cells exactly share a facet,
and each region from $\Pi$ consists of a set of adjacent cells. Such a
partition can be constructed by employing cell decomposition algorithms used
in motion planning and computational geometry, {\it e.g.} one can use a
constraint triangulation \cite{HabColSchup06} or a polytopal partition
\cite{KB-TAC08-LTLCon}. Let us denote the set of partition elements by
$C=\{c_1,c_2,\ldots ,c_{|C|}\}$. For a clear understanding, Fig.
\ref{fig:partition} presents a triangular partition obtained for the
environment from Fig. \ref{fig:environment}.  We use a triangular partition
for the case study presented throughout this paper, although our approach can
be applied to any partition scheme.

\begin{figure}
\center
\includegraphics[width=0.7\columnwidth]{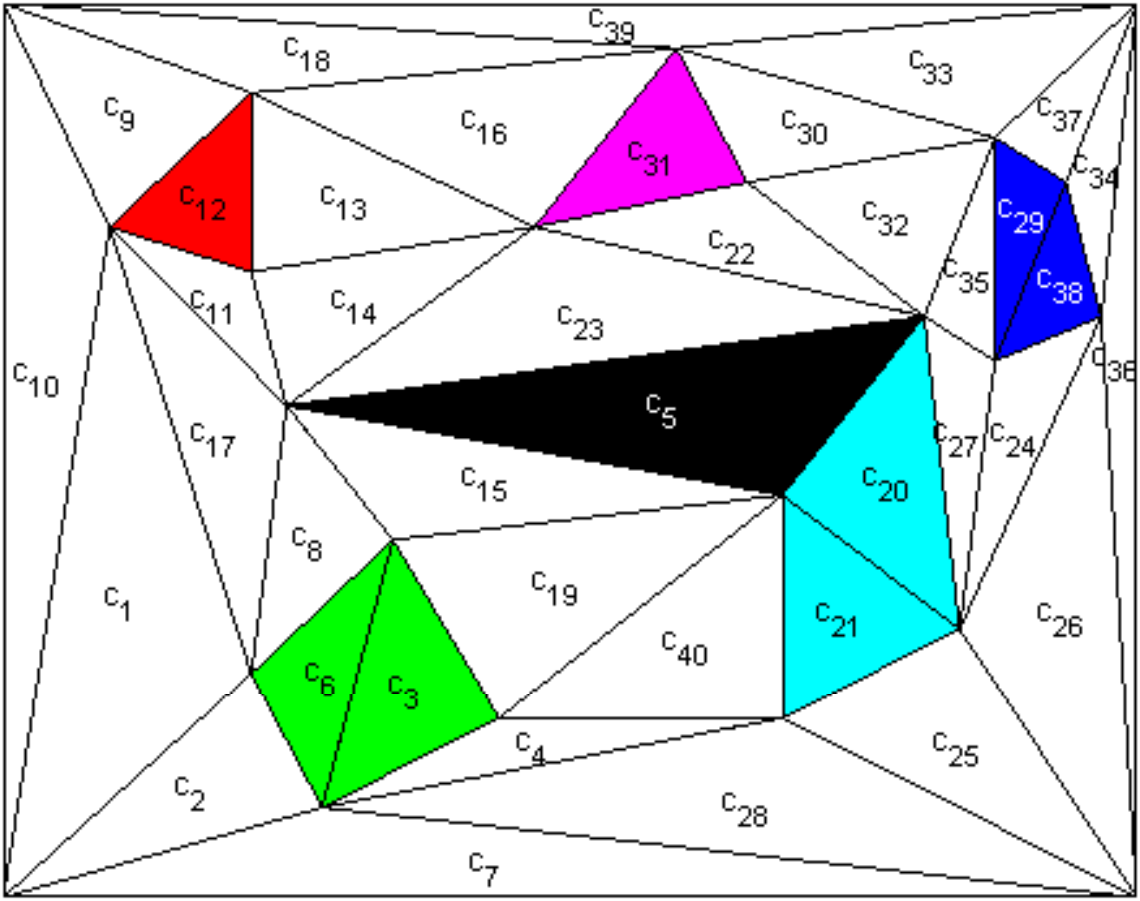}
\caption{Triangular partition consisting of 40 regions,
corresponding to the environment from Fig.
\ref{fig:environment}.} \label{fig:partition}
\end{figure}

The second step is to reduce each unicycle with kinematics
(\ref{eqn:unicycle_kinematics}) to a fully-actuated point robot placed in
unicycle's reference point. We use the approach from \cite{DOK-ICRA98}, where
a non-singular map relates the velocity $u$ of the reference point to the
initial controls $w=[v,\omega]^T$. Note that $u$ can be conservatively
bounded by a polyhedral set $U$, such that the resulted control $w$ is in
$W$.

\begin{definition}\label{def:tr_sys_rob_abstr}
The transition system abstracting the motion capabilities of unicycle $i$,
$i=1,\ldots ,n$ has the form
$T_i=(Q_i,q_{0i},\rightarrow_i,\Pi\cup\{\emptyset\},\rho)$, where:
\begin{itemize}
\item $Q_i=C$, {\it i.e.} the set of states is given by the cells from
    partition; \item The initial state $q_{0i}\in C$ is the cell where
    the reference point of unicycle $i$ is initially deployed;
\item The transition relation $\rightarrow_i\in C\times C$ is created as
    follows:
\begin{itemize} \item $(c_i,c_i)\in
\rightarrow_i$ if we can design a feedback control law making cell
$c_i$ invariant with respect to the trajectories of the reference
point of unicycle $i$, and \item $(c_i,c_j)\in \rightarrow_i$, $i\neq
j$ if $c_i$ and $c_j$ are adjacent and we can design a feedback
control law such that the reference point of unicycle $i$ leaves cell
$c_i$ in finite time, by crossing the common facet of $c_i$ and
$c_j$;
\end{itemize}
\item $\Pi$ labels the set of regions of interest, and symbol $\emptyset$
    corresponds to the space not covered by any region of interest; \item
    The observation map $\rho$ associates each cell from the partition
    with the corresponding proposition from $\Pi$, or with the symbol
    $\emptyset$.
\end{itemize}
\end{definition}

Considering the unicycles reduced to their reference point, the construction
of the continuous controllers corresponding to the transition relation from
Definition \ref{def:tr_sys_rob_abstr} is done by using results for facet
reachability and invariance in polytopes \cite{HabColSchup06}. We just
mention that designing such feedback control laws reduces to solving a set of
linear programming problems in every cell from partition, where the
constraints result from the control bounds $U$ and the considered adjacent
cells. Also, since the reference point is fully-actuated and the control
bounds $U$ include the origin, we obtain a transition between every adjacent
cells, as well as a self-loop in every state of $T_i$. Therefore, a run of
$T_i$ can be implemented by unicycle $i$ by imposing specific control laws
for the reference point in the visited cells, and by mapping these controls
to $w$. We note that, since the unicycles are identical, the only difference
between transition systems $T_i$ is given by their initial states.

{\bf Case study revisited:} The partition from Fig. \ref{fig:partition}
enables us to construct a transition system with 40 states corresponding to
each robot, where the transitions are based on adjacency relation between
cells from environment, and the observations are given by the satisfied
region. Fig. \ref{fig:vector_field_tr_sys} illustrates some vector fields
obtained from driving-to-facet control problems and from invariance
controlled design, as well as the corresponding transitions from system
$T_i$.
\endproof

\begin{figure}
\center
\includegraphics[width=0.6\columnwidth]{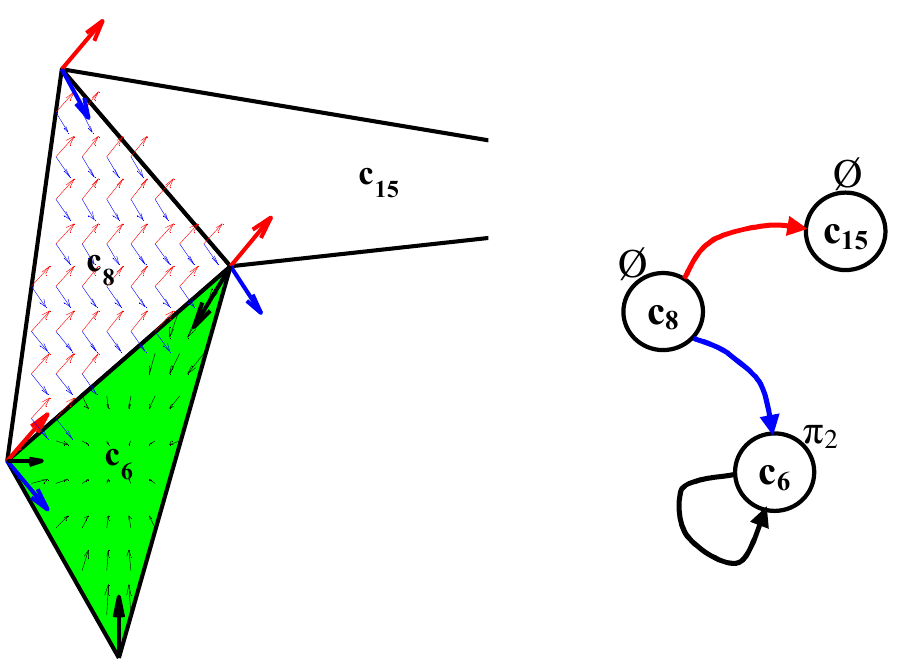}
\caption{Left: vector fields driving any initial state from cell
$c_8$ to a neighbor (blue and red colors), and vector field
making $c_6$ invariant (black). Right: transitions from $T_i$ are colored in accordance with the colors of the
created vector fields. Observation given by map $\rho$ is placed
near each state.} \label{fig:vector_field_tr_sys}
\end{figure}

\subsection{Satisfying Behavior of the Team}\label{sec:team_run}

In this part of the solution, we use ideas from
\cite{KB-ICNSC-06,KB-TRO-2009} to find a run (for the whole team) satisfying
the formula $\phi$. The $n$ transition systems $T_i$ are combined into a
global one, $T_G$, capturing the synchronized motion of the team
(synchronized in the sense that robots change at the same time the occupied
cells from partition, or states from $T_i$). Then, by using model checking
inspired techniques as mentioned in Sec. \ref{sec:prelim}, we find a run for
the whole team, in the form of a prefix-suffix sequence of tuples.

\begin{definition}\label{def:global_tr_sys}
The transition system $T_G=(Q_G,q_{G0},\rightarrow_G,\Pi,\rho_G)$ capturing
the behavior of the group of $n$ robots is defined as the synchronous product
of all $T_i$'s, $i=1,\ldots ,n$:
\begin{itemize}
\item $Q_G = Q_1\times\ldots\times Q_n$, \item
    $q_{G0}=(q_{01},\ldots,q_{0n})$, \item $\rightarrow_G\subset
    Q_G\times Q_G$ is defined by
    $\left((q_1,\ldots,q_n),(q'_1,\ldots,q'_n)\right)\in \rightarrow_G$
    if and only if $(q_i,q'_i)\in\rightarrow_i$, $i=1,\ldots,n$, \item
    $\Pi$ is the observation set, \item $\rho_G:Q_G\rightarrow 2^\Pi$ is
    defined by $\rho_G\left((q_1,\ldots,q_n)\right)=\cup_{i=1}^n
    \{\rho(q_i)\}$.
\end{itemize}
\end{definition}

We now find a run $R$ of $T_G$ such that the generated word $\rho_G(R)$
satisfies $\phi$. For this, we use the tool from \cite{KB-TAC08-LTLCon}, and
we impose the optimality criterion that during the prefix and one iteration of
the suffix, the total number of movements between partition cells is
minimized. This is accomplished by adding weights to transitions of $T_G$,
where the weight of a transition equals the number of robots that change
their state from $T_i$. These weight are inherited when taking the product of
$T_G$ with $B_\phi$.   This optimality criterion minimizes the memory used on robots for storing motion controllers
(feedback control laws driving robots from one cell to an adjacent one).

In \cite{KB-ICNSC-06,KB-TRO-2009}, such a global run was projected to
individual runs of robots. Then, the robots can
be controlled by the affine feedback controllers that map to
unicycle controls from Sec. \ref{sec:rob_abstraction}. In deployment, when the team makes a transition from one tuple of the run to the next, the robots must synchronize (communicate) with each other and wait until every member finishes the previous transition. The synchronization will occur on the boundaries of the cells when crossing from a cell to another.

\begin{remark}\label{rem:reduce_T_G}
Since the robots are identical, the number of states from $T_G$ can be reduced by designing a bisimilar
(equivalent) transition system (the quotient induced by robot permutations)
\cite{KB-ICNSC-06}. Thus, the computation complexity of finding a run $R$ is
manageable even for large teams.
\end{remark}

{\bf Case study revisited:} For the above introduced example (Fig.
\ref{fig:environment} and \ref{fig:partition}, and specification
(\ref{eqn:specification})), we obtain a team run $R = \textrm{prefix, suffix, suffix},
\ldots$, with 7 states in prefix and 8 states in suffix, as shown in Eqn.
(\ref{eqn:run_team_example}).
\begin{equation}\label{eqn:run_team_example}
\begin{split}
\textrm{prefix} = &\left( \begin{array}{c} c_{7} \\ c_{4} \\ c_{28}
\end{array} \right)\,
\left( \begin{array}{c} c_{2} \\ c_{3} \\ c_{28}
\end{array} \right)\,
\left( \begin{array}{c} c_{1} \\ c_{6} \\ c_{28}
\end{array} \right)\,
\left( \begin{array}{c} c_{10} \\ c_{8} \\ c_{28}
\end{array} \right)\,\\
&\left( \begin{array}{c} c_{9} \\ c_{17} \\ c_{25}
\end{array} \right)\,
\left( \begin{array}{c} c_{18} \\ c_{11} \\ c_{26}
\end{array} \right)\,
\left( \begin{array}{c} c_{16} \\ c_{11} \\ c_{24}
\end{array} \right)\,\,\,\\
\textrm{suffix} = &\left( \begin{array}{c} c_{31} \\ c_{12}
\\ c_{38} \end{array} \right)\,
\left( \begin{array}{c} c_{31} \\ c_{11} \\ c_{38}
\end{array} \right)\,
\left( \begin{array}{c} c_{31} \\ c_{17} \\ c_{24}
\end{array} \right)\,
\left( \begin{array}{c} c_{31} \\ c_{8} \\ c_{27}
\end{array} \right)\,\\
&\left( \begin{array}{c} c_{31} \\ c_{6} \\ c_{20}
\end{array} \right)\,
\left( \begin{array}{c} c_{31} \\ c_{8} \\ c_{27}
\end{array} \right)\,
\left( \begin{array}{c} c_{31} \\ c_{17} \\ c_{27}
\end{array} \right)\,
\left( \begin{array}{c} c_{31} \\ c_{11} \\ c_{24}
\end{array} \right)
\end{split}
\end{equation}

Run $R$ means that the robots start from their initial cells, and the first
to robots evolve to cells $c_2$ and $c_3$ respectively. The synchronization
imposed by the construction of $T_G$ and used in
\cite{KB-ICNSC-06,KB-TRO-2009} would require that the first two robots cross
from $c_7$ to $c_2$ and from $c_4$ to $c_3$ synchronously (at exactly the
same time), and so on. By construction, the word generated by $R$ over the
set of propositions $\Pi$, $\rho_G(R)$, satisfies formula $\phi$ from Eqn.
(\ref{eqn:specification}). However, the mentioned synchronizations are
disadvantageous, because they require a lot of communication and waiting, and
they imply many discontinuities is the control input of each robot (due to
the frequent stops at the facets of traversed cell). Intuitively, one can
observe that synchronizations along prefix of $R$ for example do not
contribute to the satisfaction of the formula, nor they would lead to its
violation, because the robots just head towards some regions they have to
visit.\endproof

As evident from the case study, deployment strategies for a satisfying run
$R$ as proposed in \cite{KB-ICNSC-06,KB-TRO-2009} require a significant
amount of synchronization (and thus communication) among the team. In this
paper, rather than synchronizing robots for each transition in $R$, we aim in
finding a reduced set of transitions along $R$ that require synchronization.
This problem is formalized in Sec. \ref{sec:minim_comunic_prob}. Its solution
is given in Sec. \ref{sec:find_synch_mom}, in the form of an algorithmic
procedure returning a reduced number of necessary synchronization moments
along $R$, together with deployment strategies for robots. We note that in
\cite{KB-TRO-2009}, we developed an algorithmic tool that tests if
unsynchronized motion of the team can lead to a violation of the formula.  If
the answer is yes (such is the case for the case study considered in this
paper), then synchronization required in \cite{KB-TRO-2009} cannot be
reduced.  Therefore, the solution to the problem formulated in Sec.
\ref{sec:minim_comunic_prob} constitutes a significant improvement over
\cite{KB-TRO-2009}.

\subsection{Minimizing Communication}\label{sec:minim_comunic_prob}

\begin{problem}\label{pr:synchronization}
Given a run $R$ of $T_G$ that satisfies the LTL specification $\phi$, find a
team control and communication strategy that requires a reduced number of
inter-robot synchronizations than in the synchronization-based deployment,
while at the same time guaranteeing that the produced motion of the team
satisfies $\phi$.
\end{problem}

Central to our approach to Problem \ref{pr:synchronization} is an algorithm
that takes as input the satisfying run $R=R(1)R(2)R(3)\ldots$ and returns a
reduced number of necessary {\it synchronization moments}, where the $i^{\rm
th}$ moment along R is defined as the index	 $i$ corresponding to $R(i)$.
Motivated by the fact that synchronization by stopping and waiting at region
boundaries is not always necessary to produce a desired tuple, we consider
two types of synchronizations: in a {\it weak} synchronization, a certain
tuple is generated because there exists an instant of time at which the
robots are in the corresponding cells; a {\it strong} synchronization ensures
that a sequence of two successive tuples from $R$ is observed.  Note that a
strong synchronization at each moment in $R$ is exactly the stop and wait
strategy from \cite{KB-ICNSC-06,KB-TRO-2009}. Since the run is given in the
prefix-suffix form, our algorithmic framework will return a finite set of
moments that require synchronization, as well as the type for each
synchronization moment.

\section{SOLUTION TO PROBLEM \ref{pr:synchronization}}\label{sec:find_synch_mom}

This section provides a solution to Problem \ref{pr:synchronization}. We
first construct an algorithmic procedure for finding a set of necessary
synchronization moments (Sec. \ref{sec:algorithm_synch_mom} -
\ref{sec:proof}), and then we present a communication strategy guaranteeing
that the synchronization moments are satisfied (Sec.
\ref{sec:communic_strategy}).

Without loss of generality, we assume that run $R=R(1)R(2)R(3)\ldots$ is in
the prefix-suffix form (see Sec. \ref{sec:prelim}). Assume that the prefix
has length $k-1$ and the suffix has length $l-k+1$, with $R(j)=\left(c^1_j,
c^2_j, \ldots , c^n_j\right)^T$, for $j=1,\ldots ,l$ (and $R(l+1)=R(k)$,
$R(l+2)=R(k+1)$, $\ldots$). For avoiding supplementary notations implying
repetitions of the suffix, we assume that whenever an index along $R$ exceeds
$l$, that index is automatically mapped to the set $\{k,\ldots ,l\}$, {\it
i.e.} if $j=l$, then index $j+1$ is replaced with $k$, and so on. For an
easier understanding, we use superscripts for identifying a robot and
subscripts for indexing the cells.

For any robot $i=1,\ldots ,n$ and for any moment $j=1,\ldots ,l-1$, cells
$c^i_{j}$ and $c^i_{j+1}$ are either adjacent or identical, and the same is
true for cells $c^i_{l}$ and $c^i_{k}$. Recall that from the abstraction
process of continuous robot trajectories, $R$ does not contain any successive
and finite repetition of a $n$-tuple.

\subsection{Finding Synchronization Moments}\label{sec:algorithm_synch_mom}

The idea of constructing a solution to Problem
\ref{pr:synchronization} is to start with no synchronization
moments, and iteratively test if the formula can be violated and
update the set of synchronization moments and their type.

Let $S\subseteq\{1,\ldots ,l\}$ be an arbitrary set of synchronization
moments, and let us impose the type of each synchronization moment by
creating a map $\tau:S\rightarrow \{weak,strong\}$, where $\tau(j)=weak$
means a {\it weak} synchronization at position $j$, and $\tau(j)=strong$
means a {\it strong} synchronization at position $j$, $\forall j\in S$.

As mentioned in Sec. \ref{sec:minim_comunic_prob}, a {\it weak}
synchronization at moment $j$ along run $R$ means that the tuple $R(j)$ is
reached by the robots, {\it i.e.,} there is a moment when the robots are in
cells $c^1_j,\, c^2_j, \ldots, c^n_j$, respectively. A {\it strong}
synchronization at moment $j$ along run $R$ means that there is a weak
synchronization at $j$, and additionally the robots synchronously enter the
next tuple ($R(j+1)$). In other words, all moving robots $i$ cross from cells
$c^i_j$ to cells $c^i_{j+1}$ at the same time.

One can observe that a strong synchronization at position $j$ is not
equivalent to two weak synchronizations at $j$ and $j+1$. The strong
synchronization guarantees that in the generated team run the tuple $R(j)$ is
immediately followed by the tuple $R(j+1)$. However, the weak moments
guarantee that the $R(j)$ and $R(j+1)$ tuples are generated, but there may
appear different tuples between them. It can be noted that a weak
synchronization at moment $j$ is equivalent with a strong one at the same
moment if and only if $j=k$ and the suffix of $R$ has length 1 ($k=l$).

For testing the correctness of a set of synchronization moments, we developed
a procedure $test\_feasibility_\phi(R,S,\tau)$, which takes as inputs the
formula $\phi$, the run $R$, a set $S$ of synchronization moments and a map
$\tau$. The returned output is either ``feasible" (set $S$ with map $\tau$
guarantees the satisfaction of the formula, no matter how the robots move in
between synchronization moments) or ``not feasible" (it is possible to
violate the formula by imposing just the moments from $S$ with type $\tau$).
We postpone the details on $test\_feasibility_\phi(R,S,\tau)$ until Sec.
\ref{sec:test_feasib}.

We use Algorithm \ref{alg:synch_moments} for obtaining a solution to Problem
\ref{pr:synchronization}, in the form of a set $S$ of synchronization moments
and a map $\tau$. The intuition behind this algorithm is to start with no
synchronization moment ($S=\emptyset$) and increase $S$ until we obtain a
feasible set together with a corresponding map $\tau$. We show the
correctness of the solution given by Algorithm \ref{alg:synch_moments} in
Sec. \ref{sec:proof}, together with more informal explanations on the
provided pseudo-code.

\begin{algorithm}
\caption{Solution to Problem \ref{pr:synchronization}}
{\bf Inputs:} Run $R$, formula $\phi$\\
{\bf Outputs:} Set $S$, map $\tau$
\begin{algorithmic}[1]\label{alg:synch_moments}
\FOR{$synch\_type\in\{weak,strong\}$}
\STATE $S=\emptyset$, $\tau$ undefined
\STATE $lower_{bound}=1$, $moment=l$
\WHILE{$moment\geq lower_{bound}$}
    \IF{$test\_feasibility_\phi(R,S,\tau)$ = ``feasible"}
%        \STATE Synchronization moments and their type were found
        \STATE {\bf Return} set $S$ and map $\tau$
    \ENDIF
    \STATE $S_{temp}=S\cup \{moment, moment+1,\ldots , l\}$
    \STATE $\tau_{temp}(i)=\tau(i),\, \forall i\in S$
    \STATE $\tau_{temp}(i)=synch\_type,\, \forall i\in \{moment+1,\ldots , l\}$
    \FOR{$\tau_{temp}(moment)\in\{weak,strong\}$}
        \IF{$test\_feasibility_\phi(R,S_{temp},\tau_{temp})$ = ``feasible"}
            \STATE $S:=S\cup \{moment\}$
            \STATE $\tau(moment)=\tau_{temp}(moment)$
            \STATE $lower_{bound}=moment$
            \STATE $moment=l$
            \STATE {\bf Break} ``for" loop on $\tau_{temp}$
        \ELSE
            \STATE $moment:=moment-1$
        \ENDIF
    \ENDFOR
\ENDWHILE
\ENDFOR
\end{algorithmic}
\end{algorithm}

\begin{remark}[Complexity]
Algorithm \ref{alg:synch_moments} is guaranteed to finish, because in the
worst case it returns the set $S=\{1,\ldots ,l\}$, meaning that strong
synchronizations are needed at every moment (see Sec. \ref{sec:proof}).  The
worst case complexity requires $3l(l+3)/2$ iterations of the {\it
test\_feasibility} procedure.
\end{remark}

\begin{remark}[Optimality]
Algorithm \ref{alg:synch_moments} can be tailored such that it returns an
optimal solution (with respect to a cost defined by weighting different
synchronization moments). This can be done by first constructing all possible
pairs $S,\tau$ (there are $3^l$ such pairs). Then, these pairs should be
ordered according to their associated cost. Finally, the pairs should be
tested (in the found order) against the {\it test\_feasibility} procedure,
until a feasible response is obtained. Of course, the worst case would
require $3^l$ iterations of {\it test\_feasibility} (when the only solution
is $S=\{1,\ldots ,l\}$ and strong
synchronization at every moment).
\end{remark}

\subsection{Testing a Set of Synchronization Moments}\label{sec:test_feasib}

Procedure $test\_feasibility_\phi(R,S,\tau)$ follows several main steps:

\begin{enumerate}[(i)]
\item It produces an automaton $A_{R,S,\tau}$ generating all the
infinite words (sequences of observed propositions) that can
result while the robots evolve and obey synchronization moments
from $S$ (this automaton has the form of a B\"{u}chi automaton
with an observation map); \item If necessary, $A_{R,S,\tau}$ is
transformed into a standard (degeneralized) form; \item The
product automaton between $A_{R,S,\tau}$ and the B\"{u}chi
corresponding to negated LTL formula ($\mathcal{B}_{\neg\phi}$) is
computed and its language is checked for emptiness; \item If the
language is empty, the procedure returns ``feasible" and otherwise
it returns ``not feasible".
\end{enumerate}

For step (i), the run $R$ is projected to $n$ individual runs, each
corresponding to a specific robot. In each of these individual runs, we
collapse the finite successive repetitions of identical states (cells) into a
single occurrence (such repetitions mean that the individual robot stays
inside a cell). Let us denote the resulted runs by $R^i = q_1^i\,q_2^i\ldots
\left[q_{k_i}^i\ldots q_{l_i}^i\right]\ldots$, where prefix has length
$k_i-1$ ($k_i \leq k$) and suffix has length $l_i-k_i+1$ ($l_i \leq l$),
$i=1,\ldots ,n$. Together with individual projections and collapsing, we
construct a set of maps $\beta_i : \{1,2,\ldots ,l\} \rightarrow \{1,2,\ldots
,l_i\}$, $i=1,\ldots ,n$, mapping each index from run $R$ to the
corresponding index from the individual run $R_i$ ({\it e.g.}
$\beta_i(j)=\beta_i(j+1)$ if we have the same $i^{\rm th}$ element in tuples
$R(j)$ and $R(j+1)$).

Next, we obtain a generalized B{\"u}chi automaton $A_{R,S,\tau}$ (see Def.
\ref{def:buchi_generalized}) whose runs are the possible sequences of tuples
of cells visited during the team evolution. Thus, the language of
$A_{R,S,\tau}$ contains all possible sequences of elements from $2^\Pi$
observed during the team movement. Each robot moves without synchronizing
with the others, except for the moments from set $S$ with type $\tau$.

\begin{definition}
The automaton $A_{R,S,\tau}$ is defined as
$A_{R,S,\tau}=(Q_A,q_{A_0},\rightarrow_A,F_A,\Pi,\rho)$, where:
\begin{itemize}
\item $Q_A = \{q_1^1,q_2^1,\ldots ,q_{l_1}^1\} \times
    \{q_1^2,q_2^2,\ldots ,q_{l_2}^2\}\times \ldots \times
    \{q_1^n,q_2^n,\ldots ,q_{l_n}^n\}$ is the set of states,
    \item $q_{A_0} = (q_1^1, q_1^2,\ldots ,q_1^n)$ is the initial state,
    \item $\rightarrow_A: Q_A\rightarrow 2^{Q_A}$ is the transition
    relation,
    \item $F_A\subset 2^{Q_A}$ is the set of sets of final states, \item
        $\Pi$ is the observation set, \item $\rho_A : Q_A\rightarrow
        2^\Pi$ is the observation map, $\rho_A(q_1,q_2,\ldots
        ,q_n)=\cup_{i=1}^n\{\rho(q_i)\}$.
\end{itemize}
\end{definition}

%\footnote{We use superscripts for identifying the robots, and subscripts for
%denoting the index of a state in individual runs, {\it i.e.} $q_j^i$ denotes
%the $j^{\it th}$ state from $R^i$.}
The {\bf \it transition relation} $\rightarrow_A$ is defined as follows:
$\forall q,q' \in Q_A$, with $q=(q_{j_1}^1,q_{j_2}^2,\ldots ,q_{j_n}^n)$ and
$q'=({q'}_{j_1}^1,{q'}_{j_2}^2,\ldots ,{q'}_{j_n}^n)$,
$\left(q,q'\right)\in\rightarrow_A$ if and only if the following rules are
simultaneously satisfied:
\begin{enumerate}[(a)]
\item $q=q'$ if and only if $j_i=k_i$ and $k_i=l_i$, $i=1,\ldots
,n$; \item ${q'}_j^i \in \{q_j^i,q_{j+1}^i\}$ if $j\in\{1,2,\ldots
,l_i-1\}$, and ${q'}_j^i \in \{q_{l_i}^i,q_{k_i}^i\}$ if $j=l_i$,
$i=1,2,\ldots ,n$; \item if $\exists s\in S$ such that
$j_i=\beta_i(s)$ for $i\in I\subseteq\{1,\ldots ,n\}$, where $I$
is the largest possible subset of robots satisfying this
requirement, then:
\begin{enumerate}[(1)]
\item if $I\neq\{1,\ldots ,n\}$, then ${q'}_{j_i}^i = q_{j_i}^i$,
    $\forall i\in I$; \item if $I = \{1,\ldots ,n\}$ and
    $\tau(s)=strong$, then ${q'}_{j_i}^i = q_{\beta_i(s+1)}^i$,
    $\forall i\in I$.
\end{enumerate}
\end{enumerate}

Informally, requirements (a) and (b) capture a global progress/movement along
individual runs of robots, by also capturing the possible situations when
some robots advance ``slower" (from the point of view that it takes more time
for them to reach the next cell from partition). Requirement (c) restricts
transitions by assuming that the agents satisfy all synchronization moments
from $S$ with type given by $\tau$.

Before detailing the construction of $F_A$, we say that we
consider as a {\bf \it generated word} of $A_{R,S,\tau}$ any
trajectory that infinitely often visits all sets of states from
$F_A$. This definition of generated words is exactly the
definition of accepting words of generalized B\"{u}chi automata.
Moreover, $A_{R,S,\tau}$ has a structure similar to a B\"{u}chi
automaton, which has final sets of states that are infinitely
often encountered along an accepted run.

The {\bf \it set of sets of final states} of $A_{R,S,\tau}$ ($F_A$) is
created by using Algorithm \ref{alg:final_sets}. The construction from
Algorithm \ref{alg:final_sets} matches the purpose of generated runs of
$A_{R,S,\tau}$, in the sense that any run contains infinitely many revisits
to tuples from suffix of $R$ where synchronization is imposed. More details
on the construction of $A_{R,S,\tau}$ are given in Sec. \ref{sec:proof}.

\begin{algorithm}
\caption{Set of final sets of automaton $A_{R,S,\tau}$}
\begin{algorithmic}[1]\label{alg:final_sets}
\STATE $S_{suffix}=S\cap\{k,\ldots ,l\}$ \IF{$S_{suffix} = \emptyset$}
    \STATE $F_A=\{q_{k_1}^1,\ldots ,q_{l_1}^1\} \times \{q_{k_2}^2,\ldots ,q_{l_2}^2\} \times
    \ldots \times \{q_{k_n}^n,\ldots ,q_{l_n}^n\}$
\ELSE
    \STATE Assume $S_{suffix}=\{s_1,s_2,\ldots ,s_{|S_{suffix}|}\}$
    \STATE $F_A=\{F_1,F_2,\ldots , F_{|S_{suffix}|}\}$
    \FOR{$j=1,2,\ldots ,|S_{suffix}|$}
        \STATE $F_j=\{q_{\beta_1(s_j)}^1,q_{\beta_2(s_j)}^2,\ldots ,q_{\beta_n(s_j)}^n\}$
    \ENDFOR
\ENDIF
\end{algorithmic}
\end{algorithm}

Once $A_{R,S,\tau}$ is constructed, we have to check if there exists a
generated word of $A_{R,S,\tau}$ that violates the LTL formula (by satisfying
the negation of the formula). As mentioned at the beginning of this section,
this basically implies checking for emptiness the language of a product
between $A_{R,S,\tau}$ and $\mathcal{B}_{\neg\phi}$ (steps (ii)-(iv)).

Similar to finding a run as mentioned in Sec. \ref{sec:prelim}, this
emptiness checking can be done by using available software tools for normal
(degeneralized) B\"{u}chi automata. In case that $F_A$ contains more than one
set, $A_{R,S,\tau}$ has the structure of a generalized B\"{u}chi automata. If
this is the case, we first convert $A_{R,S,\tau}$ into a degeneralized form
(as mentioned in Sec. \ref{sec:prelim}), and then we construct the product
with $\mathcal{B}_{\neg\phi}$. The construction of this product is similar to
the one constructed between transition systems and B\"{u}chi automata
\cite{KB-TAC08-LTLCon}. The only difference is that the set of final states
of product equals the cartesian product between final states of
(degeneralized) $A_{R,S,\tau}$ and the final states of B\"{u}chi.

{\bf Example of constructing $A_{R,S,\tau}$:} We include here a simple
example, solely for the purpose of understanding the construction of
$A_{R,S,\tau}$. Therefore, we do not define an environment, nor we impose an
LTL formula. Consider a team of 2 robots and the following run $R$:
\begin{equation}\label{eqn:example_autom}
R = \left( \begin{array}{c} c_5 \\ c_6 \end{array} \right)\,
\left[ \left(\begin{array}{c} c_1 \\ c_2 \end{array} \right)\,
\left(\begin{array}{c} c_7 \\ c_2 \end{array} \right)\, \left(
\begin{array}{c} c_3 \\ c_4 \end{array} \right)\, \right]\,\ldots
\end{equation}

$R$ has a prefix of length 1, and a suffix of length 3 (the suffix is
represented between square brackets).

First, assume an empty set of synchronization moments, $S=\emptyset$. By
projecting $R$ to individual runs and collapsing successive identical states,
we obtain: $R^1=q_1^1 \left[q_2^1 q_3^1 q_4^1\right]$ and $R^2=q_1^2
\left[q_2^2 q_3^2\right]$, where $q_1^1=c_5$, $q_2^1=c_1$, $q_3^1=c_7$,
$q_4^1=c_3$, $q_1^2=c_6$, $q_2^2=c_2$, $q_3^2=c_4$. The obtained automaton
$A_{R,\emptyset,\emptyset}$ is given in Fig. \ref{fig:example_autom}(a),
where the initial state is $(q_1^1,q_1^2)$. This automaton is already in
degeneralized form (it has a single set of final states), because $S$ does
not contain synchronization moments along suffix. Also, note that once the
final set of $A_{R,\emptyset,\emptyset}$ is reached, it is never left. This
corresponds to the fact that both robots reach and follow their suffixes
(independently), and any possible observed sequence during the movement
corresponds to a word generated by $A_{R,\emptyset,\emptyset}$.

Now, assume a set $S=\{2,4\}$, with $\tau(2)=strong$ and $\tau(4)=weak$. This
means there is a strong synchronization at the beginning of every iteration
of suffix of $R$ (state $(c_1,c_2)$) and a weak synchronization at the end
(state $(c_3,c_4)$). The automaton $A_{R,S,\tau}$ created as described in
this subsection is given in Fig. \ref{fig:example_autom}(b). This automaton
has the same set of states and observations as $A_{R,\emptyset,\emptyset}$,
but the set of transitions is reduced because of synchronization rules.
$A_{R,S,\tau}$ is in generalized form, because it has 2 sets of final states
($F_1=(q_2^1,q_2^2)$ and $F_2=(q_4^1,q_3^2)$). They gray states become
unreachable because the reduced transitions of $A_{R,S,\tau}$. Note that any
word generated by $A_{R,S,\tau}$ infinitely often visits state
$(q_2^1,q_2^2)$ (which corresponds to the first synchronization moment from
$S$) and state $(q_4^1,q_3^2)$ (corresponding to the second synchronization
moment from $S$). Also, there is only one outgoing transition from
$(q_2^1,q_2^2)$, in accordance
 with the strong synchronization in this state. It
should be clear why in this case we have 2 sets of final states: if we had a
single set, containing both $(q_2^1,q_2^2)$ and $(q_4^1,q_3^2)$, then
$A_{R,S,\tau}$ would have been accepting words like $(q_1^1,q_1^2)\,
\left[(q_2^1,q_2^2)\, (q_3^1,q_2^2)\, (q_4^1,q_2^2)\right]\ldots$. However,
such a word would correspond to a spurious (impossible) movement of robots,
because state $(q_4^1,q_3^2)$ would never be visited, although it corresponds
to a synchronization moment. \endproof

\begin{figure}
   \center
   \begin{tabular}{cc}
         \includegraphics[width=0.45\columnwidth]{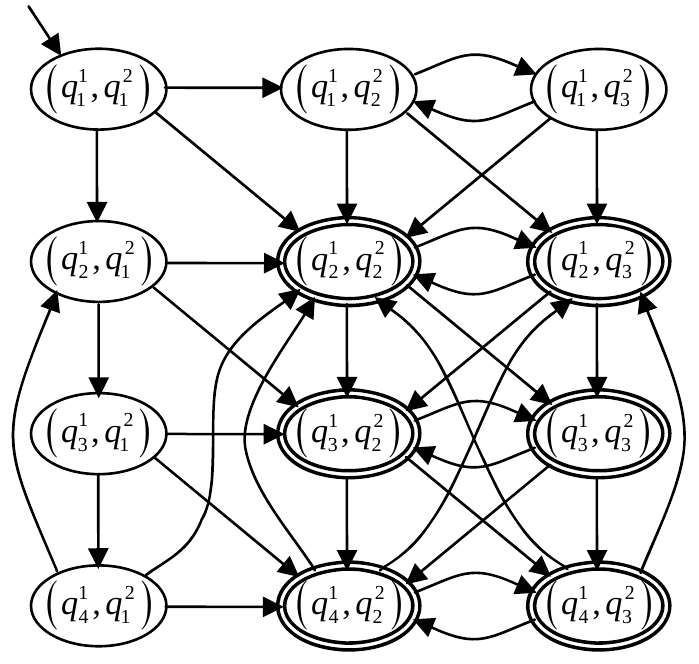} &
         \includegraphics[width=0.45\columnwidth]{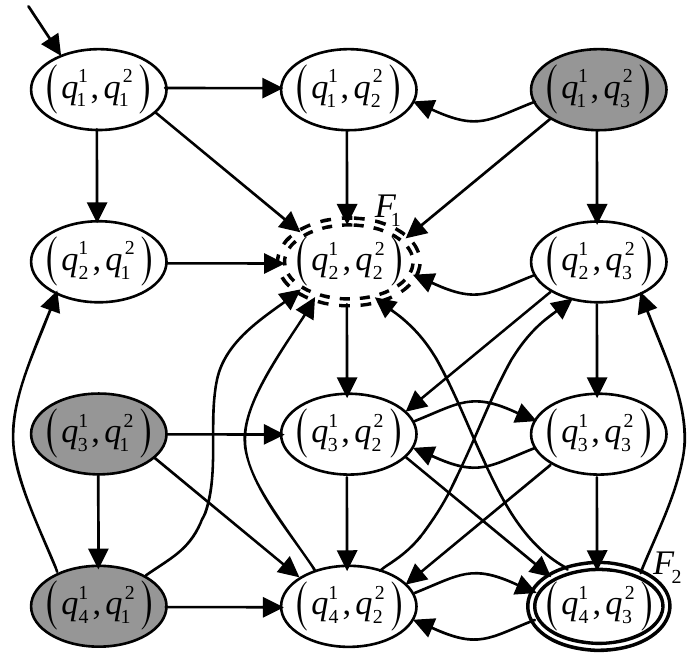}\\
         (a) & (b)\\
   \end{tabular}
\caption{Examples of two automata constructed as described in Sec.
\ref{sec:test_feasib}, from run $R$ from Eqn. (\ref{eqn:example_autom}):
(a) $A_{R,\emptyset,\emptyset}$; (b) $A_{R,\{2,4\},\tau}$, with
$\tau(2)=strong$ and $\tau(4)=weak$. Final states are double encircled, with
different line types corresponding to different final sets, and the gray
states are not reachable.} \label{fig:example_autom}
\end{figure}

\subsection{Correctness of Solution to Problem \ref{pr:synchronization}}\label{sec:proof}

\begin{theorem}\label{th:corectness_solution}
Any solution returned by Algorithm \ref{alg:synch_moments} is a feasible
solution to Problem \ref{pr:synchronization}, and Algorithm
\ref{alg:synch_moments} returns a solution in a finite number of steps.
\end{theorem}

\begin{proof}
The correctness of Algorithm \ref{alg:synch_moments} and of the
``test\_feasibility" procedure results by the construction we performed in
this section, as detailed below.

{\bf Correctness of ``test\_feasibility" procedure:} We first prove that the
``test\_feasibility" procedure cannot yield a ``feasible" output if the
inputs $S$ and $\tau$ can produce a violation of the formula. This comes from
the requirements of the transition relation $\rightarrow_A$ and from the
construction of $F_A$. Requirement (a) from definition of $\rightarrow_A$
means that $A_{R,S,\tau}$ does not self-loops, except for the case when all
individual runs have suffixes of length 1 (in this case a self-loop exists
for reiterating the suffix). Requirement (b) results from the fact that every
robot $i$ follows its individual run $R^i$ and iterates the suffix; the fact
that there is no synchronization (a moving robot might not change its current
cell, while others advance along their individual runs) is captured by the
possibility of having $q_j=q'_j$ for some robots. Thus, requirements (a) and
(b) capture a global progress/movement along individual runs of robots, by
enforcing iterations of individual suffixes. If we ignore requirement (c),
and we assume that $F_A$ is just a single set constructed as in line 3 of
Algorithm \ref{alg:final_sets}, $A_{R,S,\tau}$ can generate any possible run
resulted while robots evolve without any synchronization (set $S$ and map
$\tau$ were not yet used). Thus, if ``test\_feasibility" returns a
``feasible" answer, it means that even the unsynchronized movement is
feasible, so the motion restricted by $(S,\tau)$ would definitely imply a
satisfaction of $\phi$. However, such an approach would be very conservative,
because it does not restrict the possible generated runs of $A_{R,S,\tau}$
based on $S$ and $\tau$\footnote{A similar conservative approach was used in
\cite{KB-TRO-2009}, where only unsynchronized movements could be tested, and
synchronization moments could not be handled.}.

{\bf Conservativeness reduction of $\rightarrow_A$:} Requirement (c)
restricts transitions based on synchronization moments and their type. Thus,
(c.(1)) guarantees that if one or more robots arrived at a synchronization
moment, then they will not continue following the individual runs (by
changing their states) unless all other robots arrived at that
synchronization moment. This way, (c.(1)) captures the fact that the weak
synchronization moments and the first part of the strong ones (visiting a
certain tuple) are satisfied. Requirement (c.(2)) ensures that a tuple of
cells corresponding to a strong synchronization moment is synchronously left
by all robots that have to move at that index (the moving robots
synchronously go to the next state from their individual run, while for other
robots $\beta_i(s+1)$ implies remaining in the same state/cell). The reduced
set of transitions of $A_{R,S,\tau}$ captures the satisfaction of all
synchronization moments, and includes all possible unsynchronized
(independent) movements of robots between synchronization moments. Therefore,
the correctness of ``test\_feasibility" is not affected, while its
conservativeness is reduced.

{\bf Construction of $F_A$:} If there is no synchronization moment in the
suffix of $R$, then $F_A$ contains only one set, equal to the cartesian
product of the sets of states composing the suffixes of individual runs
$R^i$, $i=1,\ldots ,n$ (line 3 in Algorithm \ref{alg:final_sets}). This is
because no synchronization moment in suffix of $R$ means no synchronization
moments in suffixes of individual runs. Therefore, the robots independently
follow their suffixes and any $n$-tuple from set $F_A$ can be infinitely
often observed during the team evolution. In this case $A_{R,S,\tau}$ is
still too conservative, because it can generate many runs that cannot result
from the actual movement of robots ({\it e.g.} infinite surveillance of just
two states from $F_A$).

The set of accepted runs is drastically reduced if there are $x$
synchronization moments in suffix of $R$ (and also in individual suffixes).
In this case, $F_A$ will contain $x$ sets (in Algorithm \ref{alg:final_sets},
$x=|S_{suffix}|$). Each of these sets contains just one tuple, which
corresponds to one synchronization moment along individual suffixes. This
construction comes from the following aspects: (i) due to synchronization
moments along suffixes, we have additional information about some infinitely
visited states, and considering $F_A$ as in line 3 of Algorithm
\ref{alg:final_sets} would be too conservative. (ii) If there are more
synchronization moments along suffix, having a single set of final states
that contains all the corresponding tuples would be again too conservative.
This is because the transitions of $A_{R,S,\tau}$ might allow infinitely
often revisits to just a single element of that set, and we might get
spurious counterexamples when testing $A_{R,S,\tau}$ against the negation of
LTL formula. (iii) In case of a strong synchronization moment $j$ along the
suffix, adding two states (tuples corresponding to $R(j)$ and $R(j+1)$) in
the corresponding element of $F_A$ would not bring any additional benefit
than adding just $R(j)$ (as done in Algorithm \ref{alg:final_sets}). This is
because requirement (c.(2)) implies that the only possible transition from
state corresponding to $R(j)$ is to the state corresponding to $R(j+1)$.
Therefore, construction of $F_A$ as in Algorithm \ref{alg:final_sets} further
reduces the conservativeness of $A_{R,S,\tau}$ by restricting its set of
generated runs with respect to $S$ and $\tau$. $A_{R,S,\tau}$ still generates
all possible sets of tuples that the team can follow while moments from $S$
with type $\tau$ are satisfied.

There is one more step in proving the correctness of ``test\_feasibility",
namely that there exists a pair $S,\tau$ for which ``test\_feasibility"
returns a ``feasible" output. This pair is $S=\{1,\ldots ,l\}$ and
$\tau(j)=strong$, $\forall j\in S$. Indeed, in this case every state of
$A_{R,S,\tau}$ has only one outgoing transition, and the only run generated
by $A_{R,S,\tau}$ is $R$. The word generated by $R$ satisfies formula $\phi$
(because $R$ was constructed by assuming it is strongly synchronized at every
position). Therefore, the language of the product between the degeneralized
$A_{R,S,\tau}$ and $\mathcal{B}_{\neg\phi}$ is empty, and
``test\_feasibility" returns ``feasible".

{\bf Correctness of Algorithm \ref{alg:synch_moments}:} We now prove that
Algorithm \ref{alg:synch_moments} returns a feasible pair $(S,\tau)$ in a
finite number of steps. First, we explain Algorithm \ref{alg:synch_moments}
and we show that the worst case ($S=\{1,\ldots ,l\}$ and $\tau(j)=strong$,
$\forall j\in S$) is returned, if no other less restrictive pair $(S,\tau)$
was encountered.

Algorithm \ref{alg:synch_moments} starts with $S=\emptyset$ and $\tau$
undefined, and increases $S$ with at most one moment at every iteration of
the ``while" loop from line 4. For this, it goes from the last index in
suffix of $R$ (moment $l$) towards the first one, and it constructs a
temporary set of synchronization moments ($S_{temp}$). $S_{temp}$ includes
all the moments from $S$ (initially none) and all indices after the current
moment until $l$. All moments following the currently tested one are first
assumed to be weakly synchronized, and if no solution is obtained, they will
be assumed strong (line 10 and ``for" loop starting on line 1). This is
because we consider a strong synchronization more disadvantageous than a weak
one, due to the waiting and communication at borders separating adjacent
cells. The current moment is first tested with a weak synchronization, and if
no feasible answer results, it is tested with strong synchronization (loop
starting on line 11). If the current test (with $S_{temp}$ and $\tau_{temp}$)
is feasible, we add to set $S$ only the current moment, with its current
synchronization type stored in map $\tau$ (lines 13, 14). Then, on line 15 we
update the lower bound for the current synchronization moment (it doesn't
make sense to go lower than the just found moment), and we start again the
while loop on line 4 (from moment $l$ towards the lower bound). For each
disjoint assignment of ``{\it synch\_type}" (``for" loop on line 1), the
currently tested moment inside the ``while" loop from line 4 is build on the
feasible moments existing in $S$ and $\tau$ until that instant (those moments
are included in $S_{temp}$ and $\tau_{temp}$ on lines 8, 9).

Once a feasible pair $S_{temp},\tau_{temp}$ is encountered, all future
iterations from Algorithm \ref{alg:synch_moments}) try to reduce the number
of moments from $S_{temp}$ and relax their synchronization type. In the worst
case the same pair will be returned: if no other feasible pair included in
this one is found, after a number of iterations of the ``while" loop the same
pair is again encountered (but this time, the first element of the
old/feasible $S_{temp}$ is already in $S$). Now, the second element from the
old $S_{temp}$ (the first from the new $S_{temp}$) is added to $S$ and the
``while" loop is again iterated, and so on. Thus, once a feasible pair is
encountered, Algorithm \ref{alg:synch_moments} finishes in a finite number of
steps.

For proving that a feasible pair $(S_{temp},\tau_{temp})$ is ever
encountered, we show that, if no other feasible pair is found, the algorithm
eventually tests the pair where $S_{temp}=\{1,\ldots ,l\}$ and
$\tau_{temp}(j)=strong$, $\forall j\in S_{temp}$ (this pair is deemed
feasible by the ``test\_feasibility" procedure, as shown before). When
$synch\_type=weak$ (first iteration of ``for" loop on line 1), we iterate the
``while" loop for $l$ times and we do not get any feasible result. For
$synch\_type=strong$ we would get $S_{temp}=\{1,\ldots ,l\}$ and
$\tau_{temp}(j)=strong$, $\forall j\in S_{temp}$ after another $l$ iterations
of the ``while" loop. Even though we get a ``feasible" answer from
``test\_feasibility", we add just the first moment from $S_{temp}$ to $S$ (so
$S=\{1\}$) and reiterate. This time, at the $(l-1)^{\rm th}$ iteration of the
``while" loop we would get the same pair $(S_{temp},\tau_{temp})$ (when
$moment=2$). Now, $S$ is updated to $S=\{1,2\}$ and the ``while" loop is
reiterated with $lower_{bound}=2$. Since each iteration of ``while" loop
includes 3 calls to the ``test\_feasibility" procedure, the total number of
such calls is: $3\left(l\,+\,l+(l-1)+(l-2)+\ldots +1\right)=3l(l+3)/2$.

By using a similar reasoning as above, it can be easily shown that if there
exists a set $S\subseteq\{1,\ldots ,l\}$ and a map $\tau$ for which
$test\_feasibility_\phi(R,S,\tau)$=``feasible", then the pair $(S,\tau)$ will
be encountered by Algorithm \ref{alg:synch_moments} if other feasible pair is
not encountered before. This shows why Algorithm \ref{alg:synch_moments} does
not necessarily find an optimal solution (with respect to a cost defined by
weighting different synchronization moments): once a feasible pair
$(S_{temp},\tau_{temp})$ is encountered, all future iterations test only
subsets of $S_{temp}$, not other sets from $\{1,\ldots ,l\}$.
\end{proof}

\subsection{Communication and Control Strategy}
\label{sec:communic_strategy}

We complete the solution to Problem \ref{pr:synchronization} by providing a
deployment strategy such that the synchronization moments from set $S$ with
type $\tau$ are correctly implemented. A {\it weak} synchronization at moment
$j$ along $R$ is achieved by enforcing each robot $i$ to wait inside cell
$c^i_j$ (and signal this to others) until it receives a similar signal from
all other robots. A {\it strong} synchronization at moment $j$ is
accomplished by first enforcing a weak synchronization at moment $j$, and
then temporary stopping the robots for which $c^i_j \neq c^i_{j+1}$ just
before leaving cell $c^i_j$ (and entering $c^i_{j+1}$).

Since we need individual strategies, we have to adapt the set $S$ and map
$\tau$ to descriptions suitable for each robot $i$ that moves along its
individual run $R^i$, $i=1,\ldots ,n$. To solve this, for each robot $i$ we
construct a queue memory $Q^i$, where each entry contains the index along
$R^i$ when there should be enforced a synchronization, and the
synchronization type. Alg. \ref{alg:construct_queue_mem} creates these queue
memories by adding $|S|$ entries, in the ascending order of moments from $S$.

\begin{algorithm}
\caption{Queue memory $Q^i$}
\begin{algorithmic}[1]\label{alg:construct_queue_mem}
\STATE $Q^i=\emptyset$ \FORALL{$s\in S$}
    \STATE Assume that $S$ is sorted and states in $S$ are enumerated sequentially
    \STATE $moment=\beta_i(s)$
    \STATE $type=\tau(s)$
    \STATE Add entry $[moment,type]$ at the end of $Q^i$
\ENDFOR
\end{algorithmic}
\end{algorithm}

\begin{algorithm}
\caption{Individual strategy for robot $i$}
\begin{algorithmic}[1]\label{alg:individual_strategy}
\STATE $j=1$ \WHILE{TRUE}
    \STATE Read first entry $[moment,type]$ from $Q^i$
    \STATE Read second entry $[next\_moment,next\_type]$ from $Q^i$
    \IF{$moment=j$}
        \WHILE{``ready" signals not received from all other robots}
            \STATE Broadcast a ``ready to synchronize" signal
            \STATE Apply convergence controller inside current cell $q^i_j$ from run $R^i$
        \ENDWHILE
        \IF{$next\_moment\neq j$}
            \STATE Apply controller driving to the next cell $q^i_{j+1}$ until border between $q^i_j$ and $q^i_{j+1}$ is reached
        \ENDIF
        \IF{$type=strong$}
            \WHILE{``ready" signals not received from all other robots}
                \STATE Broadcast a ``ready to synchronize" signal
                \IF{$next\_moment\neq j$}
                    \STATE Stop robot
                \ELSE
                    \STATE Apply convergence controller inside current cell $q^i_j$
                \ENDIF
            \ENDWHILE
        \ENDIF
        \IF{$j<k_i$}
            \STATE Remove first entry from $Q^i$
        \ELSE
            \STATE Move first entry from $Q^i$ to the end of $Q^i$
        \ENDIF
        \IF{$next\_moment\neq j$}
            \STATE $j:=j+1$
        \ENDIF
    \ELSE
        \IF{$q^i_j\neq q^i_{j+1}$}
            \STATE Apply controller driving to the next cell $q^i_{j+1}$ until border between $q^i_j$ and $q^i_{j+1}$ is reached
            \STATE $j:=j+1$
        \ELSE
            \STATE Apply convergence controller inside current cell $q^i_j$
        \ENDIF
    \ENDIF
\ENDWHILE
\end{algorithmic}
\end{algorithm}

The queues $Q^i$ will be used by the robots in a FIFO manner, as in Alg.
\ref{alg:individual_strategy}. Each robot follows the infinite run $R^i$, by
applying correct feedback controllers and by synchronizing with other robots
when required. After fulfilling a synchronization moment, the corresponding
entry from $Q^i$ is removed or moved to the end, depending on the inclusion
of the current moment in prefix or suffix. The current index in $R^i$ is
incremented based on specific conditions, for correctly handling the
situations when two or more synchronization moments from $S$ yield the same
value through map $\beta_i$. Note that the robots not changing their cell
when a strong synchronization moment is required still synchronize twice on
that moment (for the weak and then the strong part), but they apply a
convergence controller inside current cell.

\section{CASE STUDY REVISITED}\label{sec:case_study_solved}

This section concludes the case study illustrated throughout the paper, by
applying the procedure described in Sec. \ref{sec:find_synch_mom} to the run
from Eqn. (\ref{eqn:run_team_example}). We obtain only two weak
synchronization moments, at indices 8 and 12 of run $R$ (first and fifth
positions of every repetition of suffix). This makes sense, since the
propositions satisfied by the team at the two synchronization moments are the
two sets of regions ($\{\pi_1,\,\pi_4,\,\pi_6\}$ and $\{\pi_2,\,\pi_5\}$)
that are required to be visited for the satisfaction of the formula.

The individual runs of the robots are given in Eqn.
(\ref{eqn:indiv_runs_example}), where the square brackets delimitate each
suffix, and the two weak synchronization moments are marked in bold:
\begin{equation}\label{eqn:indiv_runs_example}
\begin{array}{l}R^1 = c_{7}\,c_{2}\,c_{1}\,c_{10}\,c_{9}\,c_{18}\,c_{16}\,\left[{\bf c_{31}}\right]\,\left[{\bf c_{31}}\right]\ldots \\
R^2 = c_{4}\,c_{3}\,c_{6}\,c_{8}\,c_{17}\,c_{11}\,\left[{\bf c_{12}}\,c_{11}\,c_{17}\,c_{8}\,{\bf c_{6}}\,c_{8}\,c_{17}\,c_{11}\right]\ldots \\
R^3 = c_{28}\,c_{25}\,c_{26}\,c_{24}\,\left[{\bf
c_{38}}\,c_{24}\,c_{27}\,{\bf c_{20}}\,c_{27}\,c_{24}\right]\ldots
\end{array}
\end{equation}

The control and communication protocol from Sec. \ref{sec:communic_strategy}
points to the following deployment strategy for each robot: the robot moves
along its individual run without any communication, until it encounters a
required synchronization moment. Then, it broadcasts that it is in a ``ready"
mode for the synchronization moment, and it waits inside the current cell
until it receives ``ready" signals for the current moment from all other
robots. After this, each robot evolves again individually (without any
synchronization) until the next synchronization moment. Note that for the
above example, once robot 1 reaches its suffix it keeps applying a
convergence controller inside $c_{31}$ and broadcasting a ``ready" signal (so
that it does not induce waiting for other robots).

Some snapshots from the deployment (implemented according to Sec.
\ref{sec:communic_strategy}) are shown in Fig. \ref{fig:snapshots}, where two
repetitions of the suffix for each robot are included. A movie for the case
study is available at \url{http://hyness.bu.edu/~software/unicycles.mp4}. For
comparison, if we avoided solving Problem \ref{pr:synchronization} and
instead used the deployment strategy from \cite{KB-TRO-2009}, we would get
the team trajectory illustrated by the movie
\url{http://hyness.bu.edu/~software/unicycles-full-synch.mp4}. In this movie,
we can see that the motions of the robots are not as ``smooth'' as our
proposed approach, and the iterations for each suffix require more time. Our
approach is more suitable for real experiments as robots have less frequent
stops at region borders.

\begin{figure*}
   \center
   \begin{tabular}{cccc}
         \includegraphics[width=.22\textwidth]{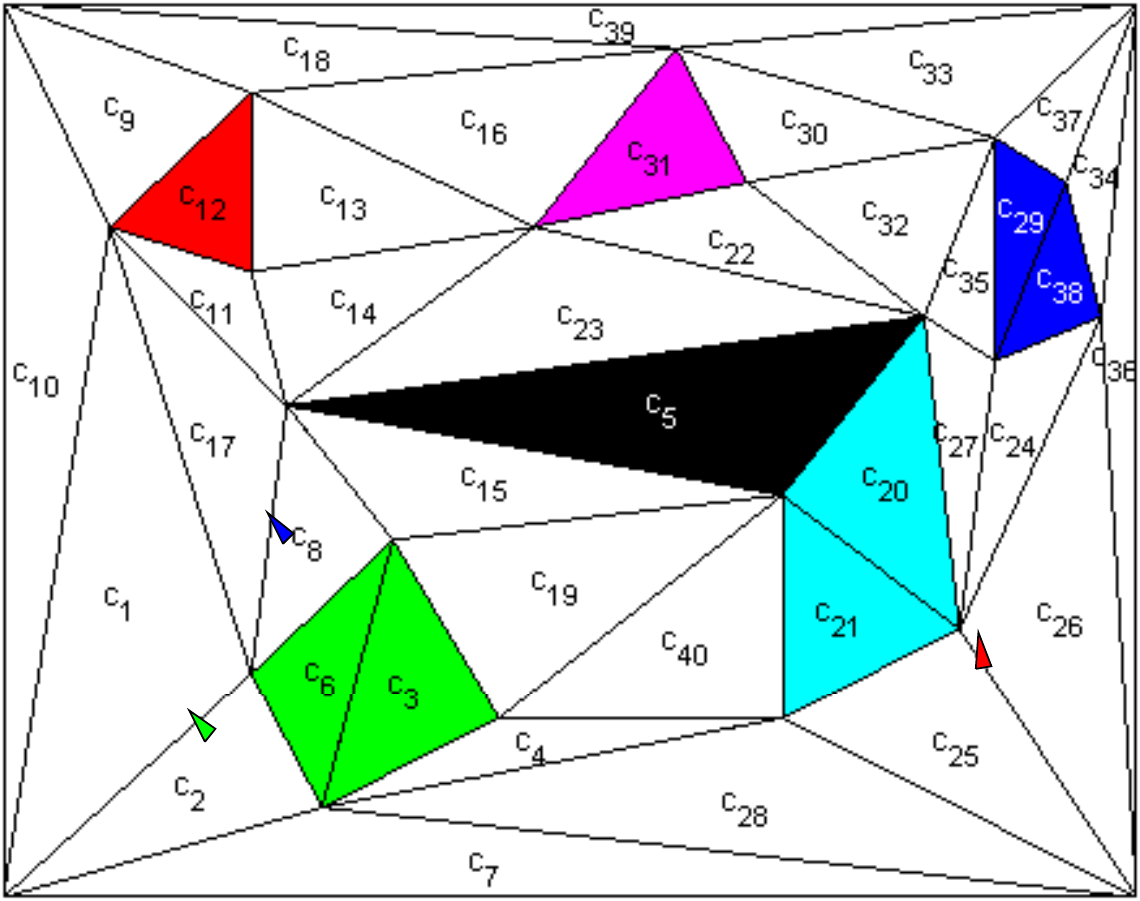} &
         \includegraphics[width=.22\textwidth]{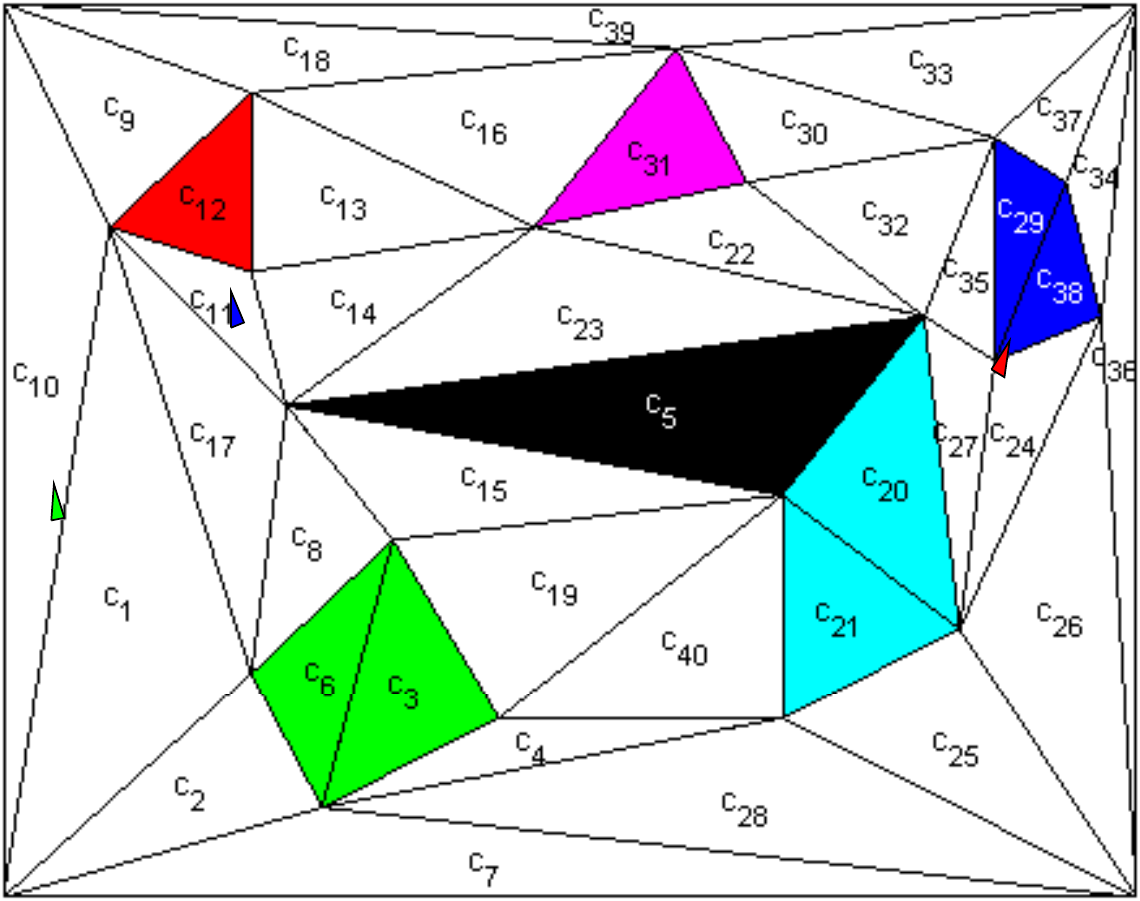} &
         \includegraphics[width=.22\textwidth]{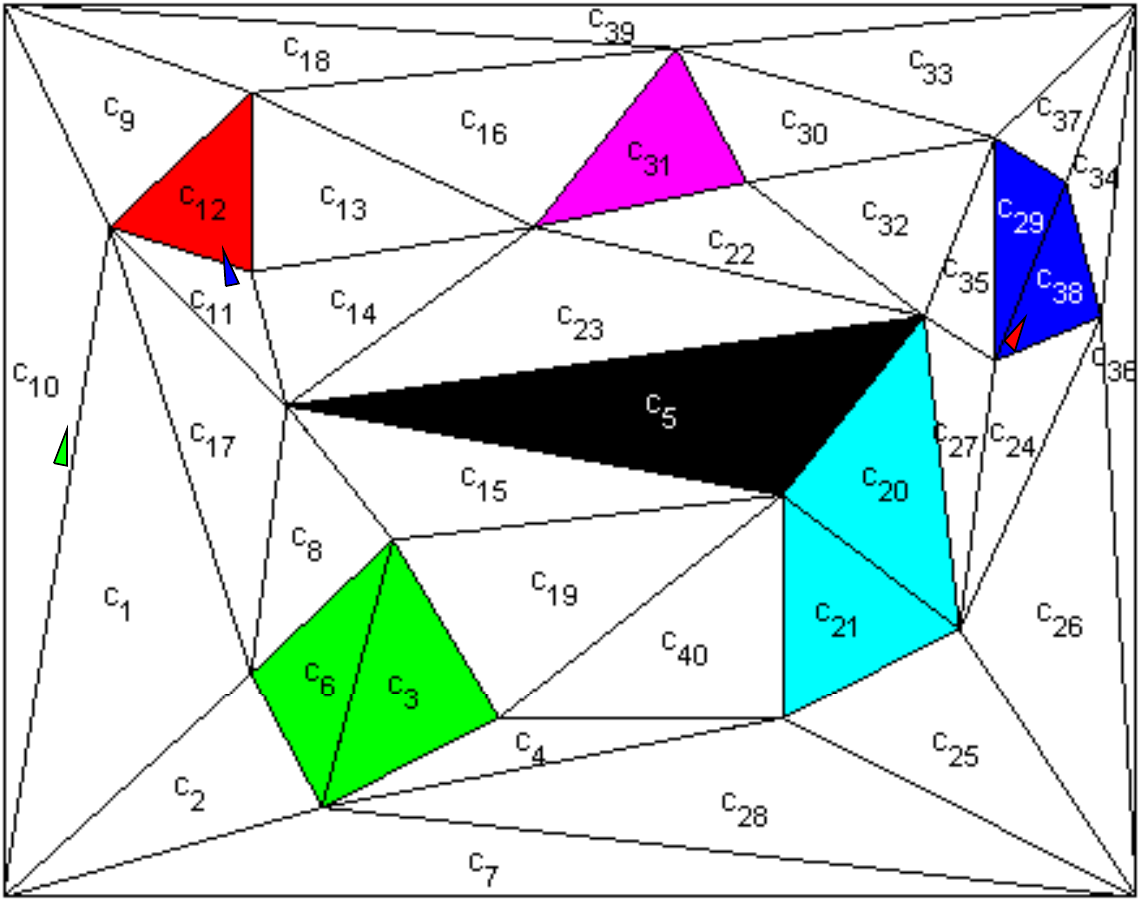} &
         \includegraphics[width=.22\textwidth]{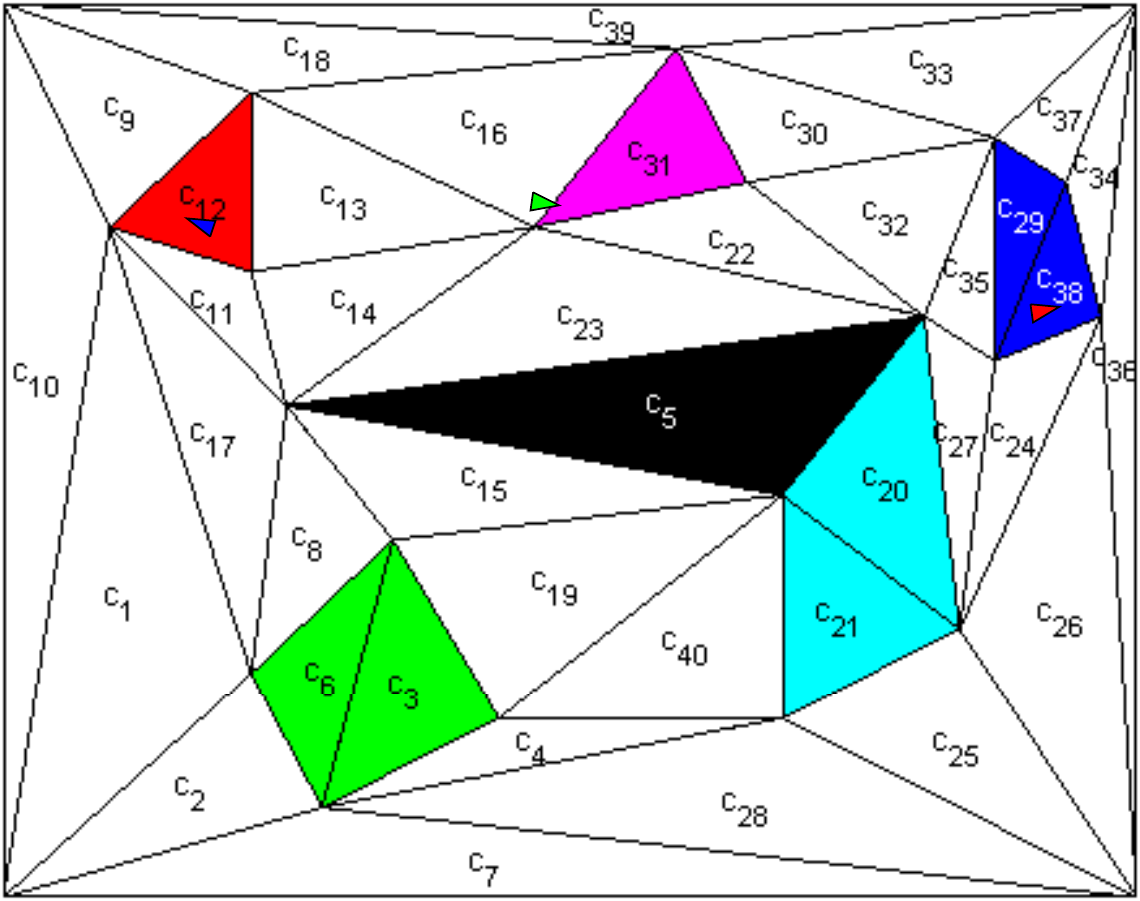}\\
         (a) & (b) & (c) &(d)\\
         \includegraphics[width=.22\textwidth]{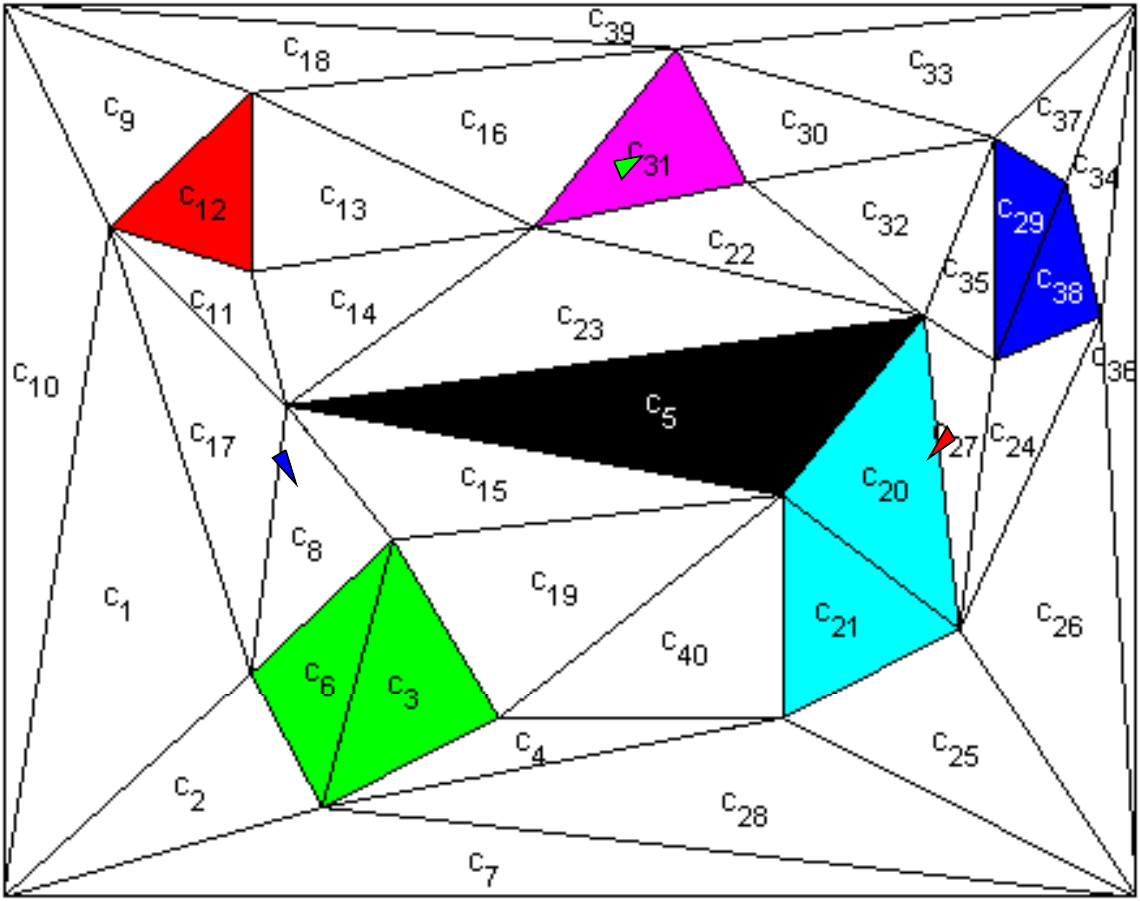} &
         \includegraphics[width=.22\textwidth]{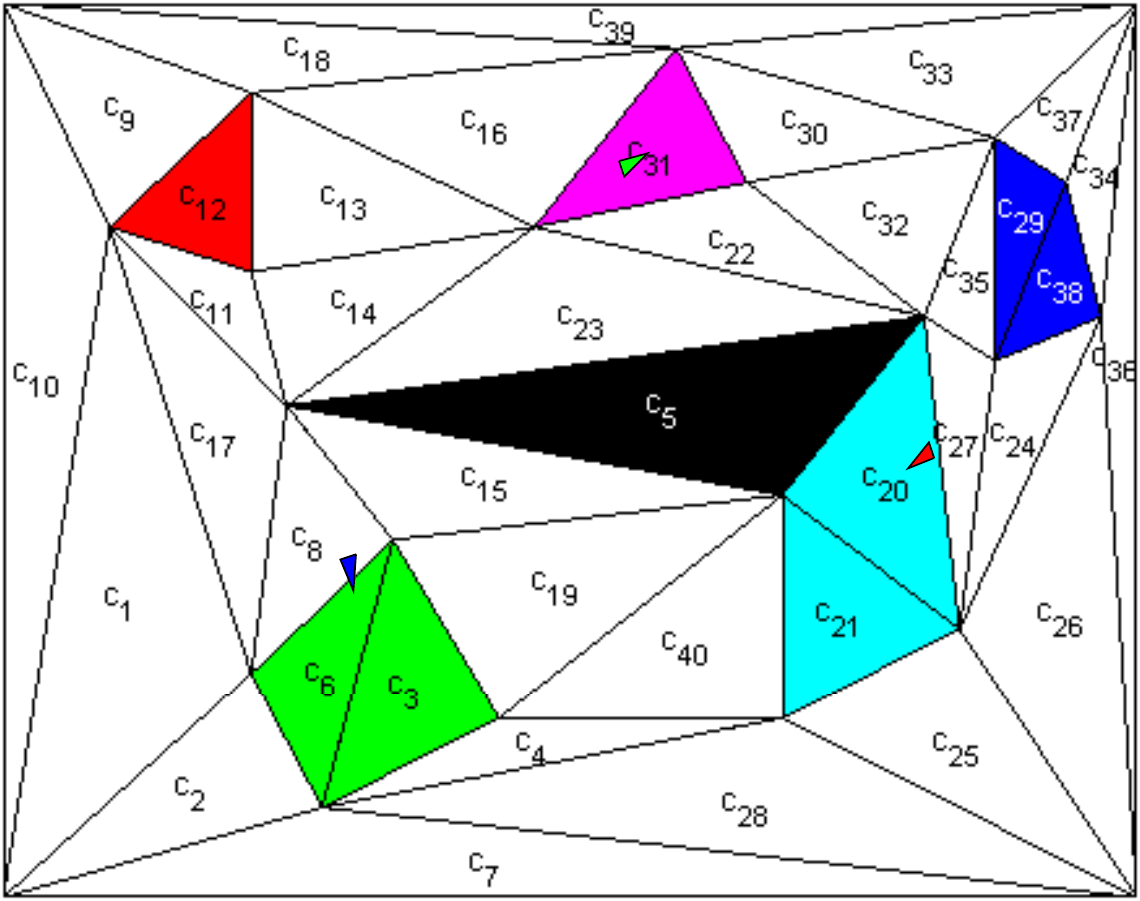} &
         \includegraphics[width=.22\textwidth]{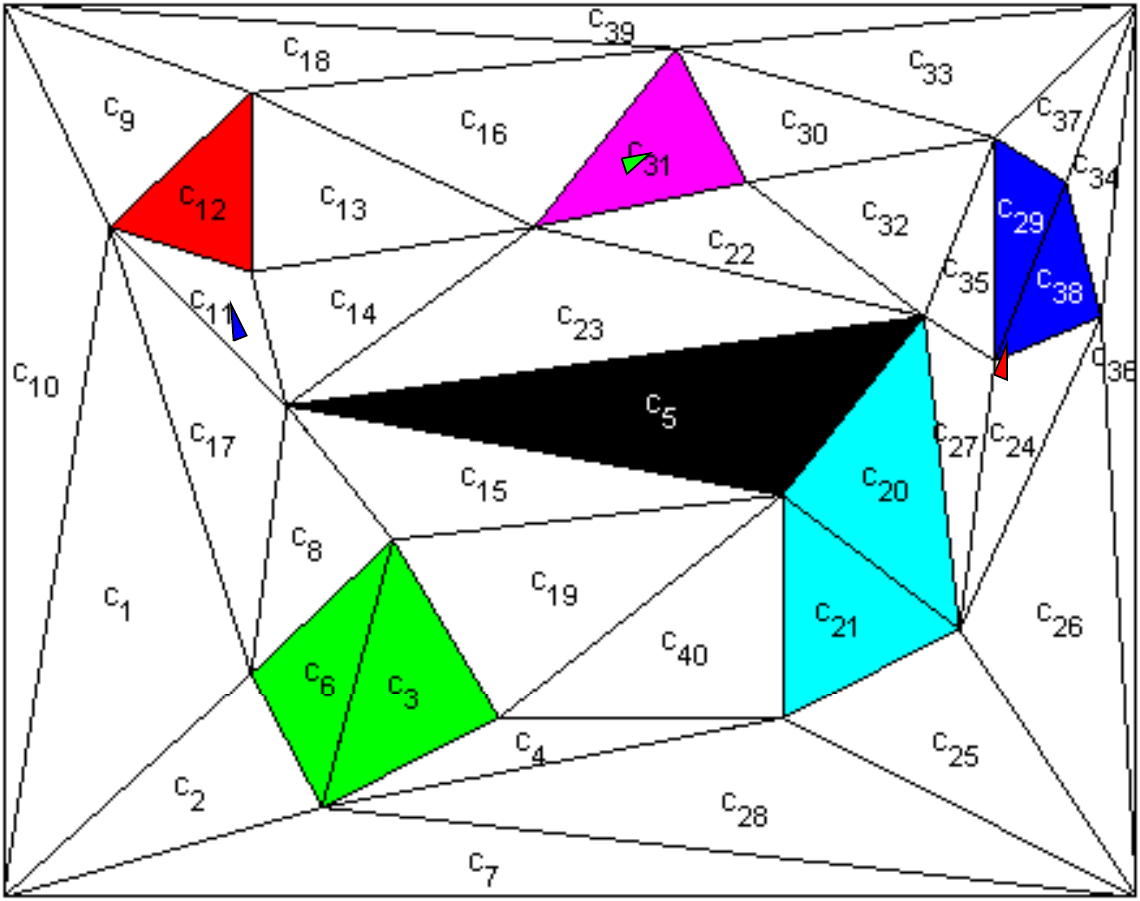}&
         \includegraphics[width=.22\textwidth]{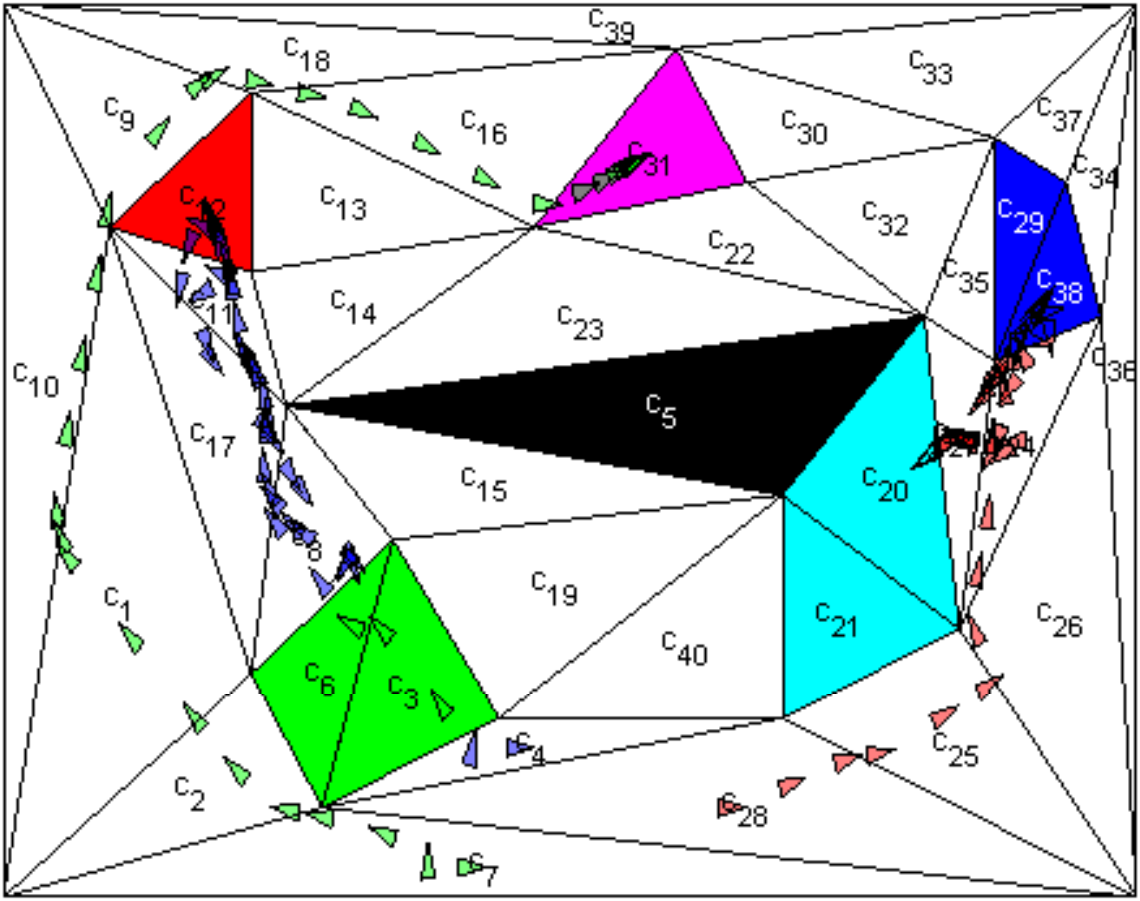}\\
         (e) & (f) & (g) &(h)\\
   \end{tabular}
\caption{Several configurations reached during movement ((a)$\ldots$(g)), and
traces of robots (h). (a) Robot 1 (green) heads to $c _{31}$, robot 2 (blue)
heads to $c_{12}$, robot 3 (red) heads to $c_{38}$; (b) Robot 3 begins to
wait for the first weak synchronization moment from suffix; (c) Robots 2 and
3 are waiting (by converging inside their regions); (d) First synchronization
moment is accomplished, all robots continue normal movement; (e) Robot 3
begins waiting for the second synchronization moment; (f) Synchronization
accomplished, robots 2 and 3 head back to $c_{12}$ and $c_{38}$ respectively;
(g) Robot 3 begins waiting for the first synchronization moment from suffix,
and when robot 2 reaches $c_{12}$ a new iteration of suffix begins.}
\label{fig:snapshots}
\end{figure*}

{\bf Computation time:} The most computationally intensive part of the
solution to Problem \ref{pr:unicycles} is finding a run $R$ (as in Sec.
\ref{sec:rob_abstraction} and \ref{sec:team_run}). For our case study, this
took about 100 minutes on a medium performance computer. In contrast, the
solution we proposed for Prob. \ref{pr:synchronization} (Sec.
\ref{sec:find_synch_mom}) took only 30 seconds. To generate a solution for
Prob. \ref{pr:synchronization}, 26 iterations of the {\it test\_feasibility}
procedure were performed until the set of synchronization moments was found.

\section{ADDITIONAL EXAMPLE}\label{sec:examples}

This section briefly presents another example on the same environment as in
Fig. \ref{fig:environment}, but considering two robots and the LTL formula:
\begin{equation}
\begin{split}
\phi &= \{\neg(\pi_3 \vee \pi_5) U (\pi_1 \wedge \pi_2)\} \wedge\\
&\{\neg(\pi_3 \vee \pi_4 \vee \pi_5) U (\pi_4 \wedge \pi_5)\}
\wedge \{\diamondsuit\square(\pi_3 \vee \pi_6)\}
\end{split}
\end{equation}

The curly brackets are only inserted to logically delimitate the
requirements: (i) black and cyan regions ($\pi_3$ and $\pi_5$) should be
avoided {\it until} red and green regions are visited, and (ii) black region
is avoided until blue and cyan regions are entered {\it at the same time},
and (iii) eventually either the black or the magenta region should be visited
and never left after that.

By considering the robots initially deployed in regions $c_1$ and
$c_{26}$ respectively, we obtained a run for the team with 12
tuples in prefix and one in suffix. Informally, first robot goes
to red region and the second one to the green one, then the robots
move toward blue and cyan regions, and after these regions are
simultaneously entered, the second robot goes to the black region
and converge there, while the first robot remains in the blue
region.

By searching synchronization moments, we obtain only two such moments: a {\it
weak} one, at the tuple when robots visit the red and green regions, and a
{\it strong} one at the tuple before visiting the blue and cyan regions. The
second synchronization moment guarantees that the blue and cyan regions are
entered at the same time. These synchronization moments automatically found
by our approach are natural at an insightful study of the requirement. The
movement of the robots is depicted in the movie available at
\url{http://hyness.bu.edu/~software/unicycles-example2.mp4}. In this movie
one can observe both synchronization moments: (1) red robot arrives in the
red region and it starts to converge there until the green robot visits the
green region; (2) green robot arrives in cell $c_{40}$ and starts to converge
there until red robot visit cell $c_{37}$ (this is the weak synchronization
part of the strong moment), and then the robots move towards blue and cyan
regions. Red robot stops at the border between $c_{37}$ and $c_{29}$ and
waits until the green robot reaches the border between $c_{40}$ and $c_{21}$
- this ensures the strong part of the synchronization (simultaneously change
occupied cells). After this, no other communication between robots is
required.

\addtolength{\textheight}{-4cm}   % This command serves to balance the column lengths

\section{CONCLUSIONS}
\label{sec:concl}

We presented a fully automated framework for deploying a team of unicycles
from a task specified as a linear temporal logic formula over some regions of
interest. The approach consists of abstracting the motion capabilities of
each robot into a finite state representation, using model checking tools to
find a satisfying run, and mapping the solution to a communication and
control strategy for each unicycle. The main contribution of the paper is the
development of an algorithmic procedure that returns a reduced set of moments
when the robots should communicate and synchronize, with the guarantee that
the specification is satisfied. A secondary contribution is the integration
of this algorithm as part of a fully automatic procedure for deployment of
teams of unicycles from specifications given as LTL formulas over regions of
interest in an environment. Future research direction includes extending this
framework to probabilistic systems such as Markov Decision Processes (MDPs)
or Partially Observable Markov Decision Processes (POMDPs), for satisfaction
of probabilistic temporal logics, such as probabilistic LTL or probabilistic
CTL.

\bibliographystyle{IEEEtran}
\bibliography{IEEEabrv,CDC_2011_bib}

\begin{thebibliography}{10}
\providecommand{\url}[1]{#1}
\csname url@rmstyle\endcsname
\providecommand{\newblock}{\relax}
\providecommand{\bibinfo}[2]{#2}
\providecommand\BIBentrySTDinterwordspacing{\spaceskip=0pt\relax}
\providecommand\BIBentryALTinterwordstretchfactor{4}
\providecommand\BIBentryALTinterwordspacing{\spaceskip=\fontdimen2\font plus
\BIBentryALTinterwordstretchfactor\fontdimen3\font minus
  \fontdimen4\font\relax}
\providecommand\BIBforeignlanguage[2]{{%
\expandafter\ifx\csname l@#1\endcsname\relax
\typeout{** WARNING: IEEEtran.bst: No hyphenation pattern has been}%
\typeout{** loaded for the language `#1'. Using the pattern for}%
\typeout{** the default language instead.}%
\else
\language=\csname l@#1\endcsname
\fi
#2}}

\bibitem{SML:06}
S.~M. LaValle, \emph{Planning Algorithms}.\hskip 1em plus 0.5em minus
  0.4em\relax Cambridge University Press, 2006, available at
  http://planning.cs.uiuc.edu.

\bibitem{SR-PN:03}
S.~Russell and P.~Norvig, \emph{Artificial Intelligence: A Modern Approach},
  2nd~ed.\hskip 1em plus 0.5em minus 0.4em\relax Prentice Hall, 2003.

\bibitem{ER-DEK:92}
E.~Rimon and D.~E. Koditschek, ``Exact robot navigation using artificial
  potential functions,'' \emph{IEEE Transactions on Robotics and Automation},
  vol.~8, no.~5, pp. 501--518, 1992.

\bibitem{SML-JJK:01}
S.~M. LaValle and J.~J. Kuffner, ``Randomized kinodynamic planning,''
  \emph{International Journal of Robotics Research}, vol.~20, no.~5, pp.
  378--400, 2001.

\bibitem{RT-IRM-MMT-JWR:10}
R.~Tedrake, I.~R. Manchester, M.~M. Tobenkin, and J.~W. Roberts, ``{LQR}-trees:
  {F}eedback motion planning via sums of squares verification,''
  \emph{International Journal of Robotics Research}, vol.~29, no.~8, pp.
  1038--1052, 2010.

\bibitem{Antoniotti95}
M.~Antoniotti and B.~Mishra, ``Discrete event models + temporal logic =
  supervisory controller: {A}utomatic synthesis of locomotion controllers,'' in
  \emph{{IEEE} Int. Conf. on Robotics and Automation}, Nagoya, Japan, 1995, pp.
  1441--1446.

\bibitem{Karaman_mu_09}
S.~Karaman and E.~Frazzoli, ``Sampling-based motion planning with deterministic
  $\mu$-calculus specifications,'' in \emph{{IEEE} Conf. on Decision and
  Control}, Shanghai, China, 2009, pp. 2222 -- 2229.

\bibitem{KB-TAC08-LTLCon}
M.~Kloetzer and C.~Belta, ``A fully automated framework for control of linear
  systems from temporal logic specifications,'' \emph{IEEE Transactions on
  Automatic Control}, vol.~53, no.~1, pp. 287--297, 2008.

\bibitem{Hadas-ICRA07}
H.~Kress-Gazit, G.~Fainekos, and G.~J. Pappas, ``Where's {W}aldo?
  {S}ensor-based temporal logic motion planning,'' in \emph{{IEEE} Int. Conf.
  on Robotics and Automation}, Rome, Italy, 2007, pp. 3116--3121.

\bibitem{Loizou04}
S.~G. Loizou and K.~J. Kyriakopoulos, ``Automatic synthesis of multiagent
  motion tasks based on {LTL} specifications,'' in \emph{{IEEE} Conf. on
  Decision and Control}, Paradise Island, Bahamas, 2004, pp. 153--158.

\bibitem{Quottrup04}
M.~M. Quottrup, T.~Bak, and R.~Izadi-Zamanabadi, ``Multi-robot motion planning:
  A timed automata approach,'' in \emph{{IEEE} Int. Conf. on Robotics and
  Automation}, New Orleans, LA, 2004, pp. 4417--4422.

\bibitem{Tok-Ufuk-Murray-CDC09}
T.~Wongpiromsarn, U.~Topcu, and R.~M. Murray, ``Receding horizon temporal logic
  planning for dynamical systems,'' in \emph{{IEEE} Conf. on Decision and
  Control}, Shanghai, China, 2009, pp. 5997--6004.

\bibitem{baier2008principles}
C.~Baier, J.-P. Katoen, and K.~G. Larsen, \emph{Principles of Model
  Checking}.\hskip 1em plus 0.5em minus 0.4em\relax MIT Press, 2008.

\bibitem{Clarke99}
E.~M. Clarke, D.~Peled, and O.~Grumberg, \emph{Model checking}.\hskip 1em plus
  0.5em minus 0.4em\relax MIT Press, 1999.

\bibitem{automata-book07}
J.~Hopcroft, R.~Motwani, and J.~D. Ullman, \emph{Introduction to Automata
  Theory, Languages, and Computation}.\hskip 1em plus 0.5em minus 0.4em\relax
  Addison Wesley, 2007.

\bibitem{Alur00}
R.~Alur, T.~A. Henzinger, G.~Lafferriere, and G.~J. Pappas, ``Discrete
  abstractions of hybrid systems,'' \emph{Proceedings of the IEEE}, vol.~88,
  pp. 971--984, 2000.

\bibitem{belta2006controlling}
C.~Belta and L.~Habets, ``Controlling a class of nonlinear systems on
  rectangles,'' \emph{IEEE Transactions on Automatic Control}, vol.~51, no.~11,
  pp. 1749--1759, 2006.

\bibitem{burridge1999sequential}
R.~Burridge, A.~Rizzi, and D.~Koditschek, ``Sequential composition of
  dynamically dexterous robot behaviors,'' \emph{The International Journal of
  Robotics Research}, vol.~18, no.~6, p. 534, 1999.

\bibitem{conner2006integrated}
D.~Conner, H.~Choset, and A.~Rizzi, ``Integrated planning and control for
  convex-bodied nonholonomic systems using local feedback control policies,''
  \emph{Proceedings of robotics: Science and systems II}, 2006.

\bibitem{DOK-ICRA98}
J.~Desai, J.~Ostrowski, and V.~Kumar, ``Controlling formations of multiple
  mobile robots,'' in \emph{Proc. IEEE Int. Conf. Robot. Automat.}, Leuven,
  Belgium, 1998.

\bibitem{HabColSchup06}
L.~Habets, P.~Collins, and J.~van Schuppen, ``Reachability and control
  synthesis for piecewise-affine hybrid systems on simplices,'' \emph{IEEE
  Transactions on Automatic Control}, vol.~51, pp. 938--948, 2006.

\bibitem{Milner89}
R.~Milner, \emph{Communication and concurrency}.\hskip 1em plus 0.5em minus
  0.4em\relax Prentice-Hall, 1989.

\bibitem{VW86}
M.~Y. Vardi and P.~Wolper, ``An automata-theoretic approach to automatic
  program verification,'' in \emph{Logic in Computer Science}, 1986, pp.
  322--331.

\bibitem{Holzmann97}
G.~Holzmann, ``The model checker {SPIN},'' \emph{IEEE Transactions on Software
  Engineering}, vol.~25, no.~5, pp. 279--295, 1997.

\bibitem{DiVinE}
J.~Barnat, L.~Brim, and P.~Ro\v{c}kai, ``{D}i{V}in{E} 2.0: {H}igh-performance
  model checking,'' in \emph{High Performance Computational Systems
  Biology}.\hskip 1em plus 0.5em minus 0.4em\relax IEEE Computer Society Press,
  2009, pp. 31--32.

\bibitem{KB-TRO-2009}
M.~Kloetzer and C.~Belta, ``Automatic deployment of distributed teams of robots
  from temporal logic motion specifications,'' \emph{IEEE Transactions on
  Robotics}, vol.~26, no.~1, pp. 48--61, 2010.

\bibitem{KB-ICNSC-06}
------, ``{LTL} planning for groups of robots,'' in \emph{IEEE International
  Conference on Networking, Sensing, and Control}, Ft. Lauderdale, FL, 2006.

\bibitem{mukund2002}
M.~Mukund, ``From global specifications to distributed implementations,'' in
  \emph{Synthesis and control of discrete event systems}.\hskip 1em plus 0.5em
  minus 0.4em\relax Kluwer, 2002, pp. 19--34.

\bibitem{yushandars}
Y.~Chen, X.~Ding, A.~Stefanescu, and C.~Belta, ``A formal approach to
  deployment of robotic teams in an urban-like environment,'' in \emph{10th
  International Symposium on Distributed Autonomous Robotics Systems (DARS
  2010)}, 2010 (to appear).

\bibitem{triangle-soft}
J.~Shewchuk, ``Triangle,'' http://www.cs.cmu.edu/ \texttildelow
  quake/triangle.html.

\bibitem{cdd-soft}
K.~Fukuda, ``cdd/cdd+ package,'' http://www.ifor.math.ethz.ch/ \texttildelow
  fukuda/cdd\texttt{\char`\_}home/.

\bibitem{Gastin01}
P.~Gastin and D.~Oddoux, ``Fast {LTL} to {B}{\"u}chi automata translation,'' in
  \emph{Conf. on Computer Aided Verification}, ser. Lecture Notes in Computer
  Science, no. 2102.\hskip 1em plus 0.5em minus 0.4em\relax Springer, 2001, pp.
  53--65.

\bibitem{Wolper83}
P.~Wolper, M.~Vardi, and A.~Sistla, ``Reasoning about infinite computation
  paths,'' in \emph{Proceedings of the 24th IEEE Symposium on Foundations of
  Computer Science}, E.~N. et~al., Ed., Tucson, AZ, 1983.

\end{thebibliography}

\end{document}